\documentclass{article}

\usepackage{arxiv}

\usepackage[utf8]{inputenc} 
\usepackage[T1]{fontenc}    
\usepackage{hyperref}       
\usepackage{url}            
\usepackage{booktabs}       
\usepackage{amsfonts}       
\usepackage{nicefrac}       
\usepackage{microtype}      
\usepackage[pdftex]{graphicx}
\usepackage{amssymb,amsfonts,amsmath,amscd,amsthm}
\usepackage[printonlyused,withpage]{acronym}
\usepackage{natbib}
\usepackage{mathrsfs}
\usepackage{rotating}
\usepackage{tikz-cd}

\graphicspath{{figs/}}

\newcommand{\bbm}{\begin{bmatrix}}
\newcommand{\ebm}{\end{bmatrix}}
\newcommand{\mbf}{\mathbf}
\newcommand{\mbs}[1]{{\boldsymbol{#1}}}

\def\om{\omega}

\newcommand{\beq}{\begin{equation}}
\newcommand{\eeq}{\end{equation}}
\newcommand{\bdis}{\begin{displaymath}}
\newcommand{\edis}{\end{displaymath}}
\newcommand{\beqn}[1]{\begin{subequations}\label{eq:#1}\begin{eqnarray}}
\newcommand{\eeqn}{\end{eqnarray}\end{subequations}}
\newcommand{\wdg}{\wedge}

\newcommand{\Wdg}{\curlywedge}

\newcommand{\Tbig}{\mbs{\mathcal{T}}}

\newcommand{\Tsmall}{\mbf{T}}

\newcommand{\Jbig}{\mbs{\mathcal{J}}}

\newcommand{\bg}{\gamma}
\newcommand{\Bg}{\mbs{\gamma}}
\newcommand{\BG}{\mbs{\Gamma}}

\newtheorem{lemma}{Lemma}
\newtheorem{theorem}{Theorem}
\newtheorem{corollary}{Corollary}
\newtheorem{remark}{Remark}

\hypersetup{%
    pdftitle={},
    pdfauthor={},
    pdfkeywords={},
    pdfsubject={},
    pdfstartview=FitH,%
    bookmarks=true,%
    bookmarksopen=true,%
    breaklinks=true,%
    colorlinks=true,%
    linkcolor=black,
    anchorcolor=black,%
    citecolor=black,
    filecolor=black,%
    menucolor=black,
    pagecolor=black,%
    urlcolor=black
}%

\title{Integral Forms in Matrix Lie Groups}

\author{
 {\normalfont Timothy D. Barfoot} \\
 Institute for Aerospace Studies / Robotics Institute \\
 University of Toronto \\
 \texttt{tim.barfoot@utoronto.ca} 
}

\date{}

\begin{document}

\maketitle
\title{Integral Forms in Matrix Lie Groups}

\begin{abstract}
Matrix Lie groups provide a language for describing motion in such fields as robotics, computer vision, and graphics. 
When using these tools, we are often faced with turning infinite-series expressions into more compact finite series (e.g., the Euler-Rodrigues formula), which can sometimes be onerous.  In this paper, we identify some useful integral forms in matrix Lie group expressions that offer a more streamlined pathway for computing compact analytic results.   Moreover, we present some recursive structures in these integral forms that show many of these expressions are interrelated.  Key to our approach is that we are able to apply the minimal polynomial for a Lie algebra quite early in the process to keep expressions compact throughout the derivations.  With the series approach, the minimal polynomial is usually applied at the end, making it hard to recognize common analytic expressions in the result.  We show that our integral method can reproduce several series-derived results from the literature.  
\end{abstract}

\section{Introduction}

Matrix Lie groups are widely used in several fields including robotics, computer vision, and graphics (see e.g., \citep{chirikjian01,stillwell08,absil09,chirikjian09,chirikjian16,sola18,gallierDifferentialGeometryLie2020,boumal23}).  Within robotics, they are employed within modelling \citep{deleuterio85,park95,lynch17}, state estimation \citep{wang06,wang08,bonnabel08,blanco10,wolfe11,long12,forster15,dellaert17,bonnabel17,mahony21,barfoot_ser24}, and control theory \citep{murray94,bullo95,sastry99}.  They provide the mathematical tools to describe important concepts such as rotations and rigid-body motions, to name a few.  For example, rotations in the plane are captured by the group, $SO(2)$, which is the set of $2 \times 2$ matrices of the form
\begin{equation}\label{eq:so2mat}
\mbf{C}(\phi) = \bbm \cos\phi & -\sin\phi \\ \sin\phi & \cos\phi \ebm,
\end{equation}
where $\phi$ is the rotation angle.  Other groups are more complex in their structure.  But where did the structure of~\eqref{eq:so2mat} come from?  One common way to define a group is begin with its Lie algebra and use the exponential map to produce all the elements of the group.  For $SO(2)$, the Lie algebra comprises all the $2 \times 2$ matrices of the form
\begin{equation}
\phi^\wdg = \phi \, \mbf{S},
\end{equation}
where $\mbf{S} = \bbm 0 & -1 \\ 1 & 0 \ebm$ is the canonical skew-symmetric matrix.  Then by using the exponential map we have that
\begin{equation}
\mbf{C}(\phi) = \sum_{m=0}^\infty \frac{1}{m!} \phi^{\wdg^m} = \mbf{1} + \phi \, \mbf{S} + \frac{1}{2}\phi^2 \mbf{S}^2 + \cdots.
\end{equation}
To arrive at the compact expression in~\eqref{eq:so2mat}, we must recognize that $\phi^\wdg$ must satisfy its own characteristic equation (i.e., the Cayley-Hamilton theorem) which means that $\phi^{\wdg^2} + \phi^2 \mbf{1} = \mbf{0}$.  Using this to reduce the powers of $\mbf{S}$ we find that
\begin{equation}
\mbf{C}(\phi) = \sum_{m=0}^\infty \frac{1}{m!} \phi^{\wdg^m} = \underbrace{\left( 1 - \frac{1}{2!} \phi^2 + \frac{1}{4!} \phi^4 + \cdots \right)}_{\cos\phi} \mbf{1} + \underbrace{\left( \phi - \frac{1}{3!} \phi^3 + \frac{1}{5!} \phi^5 + \cdots \right)}_{\sin\phi} \mbf{S} = \bbm \cos\phi & -\sin\phi \\ \sin\phi & \cos\phi \ebm,
\end{equation}
where $\mbf{1}$ is the identity matrix.
In general, each Lie algebra will have a {\em minimal polynomial}, a factor of its characteristic equation, that can be used to reduce the infinite series resulting from the exponential map to a finite number of terms with nonlinear coefficients.  The more compact expression is useful in application due to the appearance of common trigonometric functions and less computational effort required for its calculation.

As we will see later in the paper, each Lie group has a sequence of important quantities, with the usual group elements representing the first step in the sequence.  For $SO(2)$, the sequence has the form $\BG_\ell(\phi^\wdg)$ where $\BG_0(\phi^\wdg) = \mbf{C}(\phi)$.  The later elements in the sequence are typically expressed using infinite series similar to the one resulting from the exponential map.  The below commutative diagram attempts to capture the main idea of this paper:
\begin{equation}\label{eq:cd}
    \begin{tikzcd}[column sep=9 em]
        \BG_0(\phi^\wdg) = \sum_{m=0}^\infty \frac{1}{m!} \phi^{\wdg^m} \arrow[r, "\mbox{minimal polynomial}"] \arrow[d, "\mbox{integrate}"', dashed]  & \BG_0(\phi^\wdg) = \cos\phi \mbf{1} + \sin\phi \mbf{S} \arrow[d, "\mbox{integrate}"]\\
        \BG_1(\phi^\wdg) = \sum_{m=0}^\infty \frac{1}{(m+1)!} \phi^{\wdg^m} \arrow[r, "\mbox{minimal polynomial}"', dashed] \arrow[d, "\mbox{integrate}"', dashed] & \BG_1(\phi^\wdg) = \frac{\sin\phi}{\phi} \mbf{1} + \frac{1-\cos\phi}{\phi} \mbf{S}  \arrow[d, "\mbox{integrate}"] \\ 
         \BG_2(\phi^\wdg) = \sum_{m=0}^\infty \frac{1}{(m+2)!} \phi^{\wdg^m} \arrow[r, "\mbox{minimal polynomial}"', dashed] \arrow[d, "\mbox{integrate}"', dashed] & \BG_2(\phi^\wdg) = \frac{1-\cos\phi}{\phi^2}  \mbf{1} + \frac{\phi-\sin\phi}{\phi^2} \mbf{S}  \arrow[d, "\mbox{integrate}"] \\
        \mbox{} & \mbox{}
    \end{tikzcd}
\end{equation}
Our first contribution is to show that the elements in the sequence, $\BG_\ell(\phi^\wdg)$, can be expressed through a recursive integral form, meaning once we know $\BG_0(\phi^\wdg)$, for example, we can integrate for $\BG_1(\phi^\wdg)$.  Our second contribution is to realize that rather than following the dashed arrows in the above diagram, it can be more efficient to follow the solid ones.  In other words, we can apply the minimal polynomial at the $\ell = 0$ level and then simply integrate from there to produce higher-order elements in the sequence.  Naturally, we keep our development general so that it can be applied to any matrix Lie group we like, not just $SO(2)$.

The paper is organized as follows.  Section~\ref{sec:math} reviews some basics of two of the most common matrix Lie groups, $SO(3)$ and $SE(3)$; we define these groups and introduce their Jacobians and minimal polynomials.  Section~\ref{sec:bb} introduces our sequences of {\em building blocks} including $\BG_\ell(\cdot)$ but also several others of use.  We also establish the recursive integral relationship for these sequences as well as a few relationships between different types of blocks.  Section~\ref{sec:substructure} studies several specific Lie groups and delves into the blockwise substructure of each, showing that we can use smaller building blocks to construct the overall group matrices.  Finally, Section~\ref{sec:derivatives} provides results for computing temporal and partial derivatives of some of our building blocks.  Whenever possible throughout the paper, we provide example demonstrations of how our integral approach can be used to reproduce important results from the literature.  We close with some brief conclusions.

\section{Mathematical Background}\label{sec:math}

We briefly review some key concepts and notation \citep{barfoot_ser24} to prepare for what follows.  We introduce two common matrix Lie groups, $SO(3)$ (representing rotations) and $SE(3)$ (representing poses), to get things started in this section.  Readers already familiar with these groups can safely skip to the next section.  Later in the paper, we will also provide some connection to other groups used frequently in robotics.  

\subsection{Common Matrix Lie Groups}

The {\em special orthogonal group}, representing rotations, is the set of valid rotation matrices:
\begin{equation}
\label{eq:SO3}
SO(3) = \left\{  \mbf{C} \in \mathbb{R}^{3\times3} \; | \; \mbf{C} \mbf{C}^T = \mbf{1}, \mbox{det} \,\mbf{C} = 1 \right\},
\end{equation}
where $\mbf{1}$ is the identity matrix.  
It is common to map a vector, $\mbs{\phi} \in \mathbb{R}^3$, to a rotation matrix, $\mbf{C}$, through the matrix exponential,
\begin{equation}
\mbf{C}(\mbs{\phi}) = \exp\left( \mbs{\phi}^\wdg \right) = \sum_{m=0}^\infty \frac{1}{m!} \mbs{\phi}^{\wdg^m},
\end{equation}
where $(\cdot)^\wdg$ is the skew-symmetric operator,
\begin{equation}
\mbs{\phi}^\wdg = \bbm \phi_1 \\ \phi_2 \\ \phi_3 \ebm^\wdg = \bbm 0 & -\phi_3 & \phi_2 \\ \phi_3 & 0 & -\phi_1 \\ -\phi_2 & \phi_1 & 0 \ebm.
\end{equation}
The mapping is surjective-only, meaning every $\mbf{C}(\mbs{\phi})$ can be produced by many different values for $\mbs{\phi}$.  This group is {\em self-adjoint}, $\mbox{Ad}(\mbf{C}(\mbs{\phi})) = \mbf{C}(\mbs{\phi})$, which implies that we can write
\begin{equation}
\Bigl( \underbrace{\mbf{C}(\mbs{\phi}_1)}_{\rm adjoint} \, \mbs{\phi}_2 \Bigr)^\wdg =  \mbf{C}(\mbs{\phi}_1) \, \mbs{\phi}_2^\wdg \,  \mbf{C}(\mbs{\phi}_1)^T.
\end{equation}

The {\em special Euclidean group}, representing \index{poses} poses (i.e., translation and rotation), is the set of valid  transformation matrices:
\begin{equation}
\label{eq:se3}
SE(3) = \left\{  \mbf{T} = \bbm \mbf{C} & \mbf{r} \\ \;\,\mbf{0}^T & 1 \ebm \in \mathbb{R}^{4\times4} \; \Biggl| \; \mbf{C} \in SO(3), \, \mbf{r} \in \mathbb{R}^3 \right\}.
\end{equation}
It is again common to map a vector, $\mbs{\xi} \in \mathbb{R}^6$, to a transformation matrix, $\mbf{T}$, through the matrix exponential,
\begin{equation}
\mbf{T}(\mbs{\xi}) = \bbm \mbf{C}(\mbs{\phi}) & \mbf{J}(\mbs{\phi}) \mbs{\rho} \\ \mbf{0}^T & 1 \ebm = \exp\left( \mbs{\xi}^\wdg \right) = \sum_{m=0}^\infty \frac{1}{m!} \mbs{\xi}^{\wdg^m},
\end{equation}
where
\begin{equation}
\mbs{\xi}^\wdg = \bbm \mbs{\rho} \\ \mbs{\phi} \ebm^\wdg = \bbm \mbs{\phi}^\wdg & \mbs{\rho} \\ \mbf{0}^T & 0 \ebm.
\end{equation}
As is common practice \citep{barfoot_tro14}, we have broken the pose vector, $\mbs{\xi}$, into a translational component, $\mbs{\rho}$, and a rotational component, $\mbs{\phi}$.  The matrix $\mbf{J}(\mbs{\phi})$ is the left Jacobian of $SO(3)$, to be described in detail in the next section.
The mapping is also surjective-only, meaning every $\mbf{T}(\mbs{\xi})$ can be produced by many different values for $\mbs{\xi}$.

Finally, the {\em adjoint} of pose is given by
\begin{equation}\label{eq:se3adjointmap1}
\Tbig(\mbs{\xi}) = \mbox{Ad}\left( \mbf{T} \right) = \bbm \mbf{C}(\mbs{\phi}) & \left( \mbf{J}(\mbs{\phi}) \mbs{\rho} \right)^\wdg \mbf{C}(\mbs{\phi}) \\ \mbf{0} & \mbf{C}(\mbs{\phi}) \ebm,
\end{equation}
which is now $6 \times 6$.  The adjoint provides the relation
\begin{equation}
\left( \Tbig(\mbs{\xi}_1)\,  \mbs{\xi}_2 \right)^\wdg = \Tsmall(\mbs{\xi}_1) \, \mbs{\xi}_2^\wdg \, \Tsmall(\mbs{\xi}_1)^{-1}.
\end{equation}
We will refer to the set of adjoints as $\mbox{Ad}(SE(3))$.  We can map a vector, $\mbs{\xi} \in \mathbb{R}^6$, to an adjoint transformation matrix again through the matrix exponential map:
\begin{equation}
\Tbig(\mbs{\xi}) = \exp\left( \mbs{\xi}^\Wdg \right) = \sum_{m=0}^\infty \frac{1}{m!} \mbs{\xi}^{\Wdg^m},
\end{equation}
where
\begin{equation}
\mbs{\xi}^\Wdg = \bbm \mbs{\rho} \\ \mbs{\phi} \ebm^\Wdg = \bbm \mbs{\phi}^\wdg & \mbs{\rho}^\wdg \\ \mbf{0} & \mbs{\phi}^\wdg \ebm.
\end{equation}
The mapping is again surjective-only, meaning every $\Tbig(\mbs{\xi})$ can be produced by many different values for $\mbs{\xi}$.

\subsection{Jacobians}

Jacobians are required frequently when working with matrix Lie groups.  For example, for $SO(3)$ the left Jacobian expression is 
\begin{equation}\label{eq:expSO3jac}
\mbf{J}(\mbs{\phi}) = \sum_{m=0}^\infty \frac{1}{(m+1)!} \mbs{\phi}^{\wdg^m}.
\end{equation}
This Jacobian is useful in rotational kinematics, which can thus be written in one of two ways:
\begin{equation}
\dot{\mbf{C}}(\mbs{\phi}) = \mbs{\om}^\wdg \mbf{C}(\mbs{\phi}) \quad \Leftrightarrow \quad \dot{\mbs{\phi}} = \mbf{J}(\mbs{\phi})^{-1} \mbs{\om},
\end{equation}
where $\mbs{\om}$ is the angular velocity.  Considering an infinitesimal increment of time, the same Jacobian thus allows us to approximate the compounding of two exponentiated vectors $\mbs{\phi}_1, \mbs{\phi}_2 \in \mathbb{R}^3$ as
\begin{equation}
\exp( \mbs{\phi}_1^\wdg ) \, \exp( \mbs{\phi}_2^\wdg ) \approx \exp\left( \left( \mbf{J}(\mbs{\phi}_2)^{-1} \mbs{\phi}_1 + \mbs{\phi}_2\right)^\wdg \right),
\end{equation} 
where $\mbs{\phi}_1$ is assumed to be `small' \citep{klarsfeld89}.

A similar situation exists for $SE(3)$ poses, where the left Jacobian is  
\begin{equation}
\mbs{\mathcal{J}}(\mbs{\xi}) = \sum_{m=0}^\infty \frac{1}{(m+1)!} \mbs{\xi}^{\Wdg^m}.
\end{equation}
The kinematics can be written in one of two ways:
\begin{equation}
\dot{\mbf{T}}(\mbs{\xi}) = \mbs{\varpi}^\wdg \mbf{T}(\mbs{\xi}) \quad \Leftrightarrow \quad \dot{\mbs{\xi}} = \mbs{\mathcal{J}}(\mbs{\xi})^{-1} \mbs{\varpi},
\end{equation}
where $\mbs{\varpi}$ is the generalized velocity or `twist' and $\mbs{\mathcal{J}}(\mbs{\xi})$ is the Jacobian.  Again, this allows us to approximate the compounding of two exponentiated vectors $\mbs{\xi}_1, \mbs{\xi}_2 \in \mathbb{R}^6$ as
\begin{equation}
\exp( \mbs{\xi}_1^\wdg ) \, \exp( \mbs{\xi}_2^\wdg ) \approx \exp\left( \left( \mbs{\mathcal{J}}(\mbs{\xi}_2)^{-1} \mbs{\xi}_1 + \mbs{\xi}_2\right)^\wdg \right),
\end{equation} 
where $\mbs{\xi}_1$ is assumed to be `small'.  It is also known that the $SE(3)$ Jacobian can be written in terms of the $SO(3)$ Jacobian as follows:
\begin{equation}
\mbs{\mathcal{J}}(\mbs{\xi}) = \bbm \mbf{J}(\mbs{\phi}) & \mbf{Q}(\mbs{\phi}, \mbs{\rho}) \\ \mbf{0} & \mbf{J}(\mbs{\phi}) \ebm,
\end{equation}
where the off-diagonal block is \citep{barfoot_tro14}
\begin{equation}\label{eq:Q}
\mbf{Q}(\mbs{\phi}, \mbs{\rho}) = \sum_{m=0}^\infty \sum_{n=0}^\infty \frac{1}{(m+n+2)!} \mbs{\phi}^{\wdg^m} \mbs{\rho}^\wdg \mbs{\phi}^{\wdg^n}. 
\end{equation}
We will have much more to say about $\mbf{Q}(\mbs{\phi}, \mbs{\rho})$ a bit later.

\subsection{Cayley-Hamilton and Minimal Polynomials}

We will have occasion to work with series expressions of rotation and transformation matrices.  Each group has associated with it an identity that can be used to limit the number of terms in such series, related to the Cayley-Hamilton theorem.

For rotations, we have that the characteristic equation of $\mbs{\phi}^\wdg$ is 
\begin{equation}
\left|  \lambda \mbf{1} - \mbs{\phi}^\wdg \right| = \lambda ( \lambda^2 + \underbrace{\phi_1^2 + \phi_2^2 + \phi_3^2}_{\phi^2} ) = \lambda^3 + \phi^2 \lambda = 0,
\end{equation}
where $\lambda$ are the eigenvalues of $\mbs{\phi}^\wdg$ and $\phi = || \mbs{\phi} ||$.
From the Cayley-Hamilton theorem we can claim that
\begin{equation} \label{eq:ch1}
\mbs{\phi}^{\wdg^3} + \phi^2 \mbs{\phi}^\wdg = \mbf{0},
\end{equation}
since $\mbs{\phi}^\wdg$ must satisfy its own characteristic equation.  This means that we can simplify our infinite-series expressions into finite series with some trigonometric coefficients.  For example,
\beqn{CJcompact}
\mbf{C}(\mbs{\phi}) & = & \sum_{n=0}^\infty \frac{1}{n!} \mbs{\phi}^{\wdg^n} =  \underbrace{1}_{c_0(\phi)} \mbf{1} +  \underbrace{\frac{\sin\phi}{\phi}}_{c_1(\phi)}  \mbs{\phi}^\wdg + \underbrace{\frac{1-\cos\phi}{\phi^2}}_{c_2(\phi)} \mbs{\phi}^{\wdg^2}, \label{eq:Ccompact}\\
\mbf{J}(\mbs{\phi}) & = &  \sum_{n=0}^\infty \frac{1}{(n+1)!} \mbs{\phi}^{\wdg^n} =  \underbrace{1}_{j_0(\phi)} \mbf{1} + \underbrace{\frac{1-\cos\phi}{\phi^2}}_{j_1(\phi)} \mbs{\phi}^\wdg + \underbrace{\frac{\phi - \sin\phi}{\phi^3}}_{j_2(\phi)} \mbs{\phi}^{\wdg^2}, \label{eq:Jcompact}
\eeqn
are both well-known expressions, the first being the famous Euler-Rodrigues \citep{euler1770,rodrigues1840} formula for a rotation matrix.
For poses, we have that the characteristic equation of $\mbs{\xi}^\wdg$ is
\begin{equation}\label{eq:se3char}
\left|  \lambda \mbf{1} - \mbs{\xi}^\wdg \right| = \left| \begin{matrix} \lambda \mbf{1} - \mbs{\phi}^\wdg & -\mbs{\rho} \\ \mbf{0}^T & \lambda \end{matrix} \right| = \lambda^4 + \phi^2 \lambda^2 = 0.
\end{equation}
Again, from the Cayley-Hamilton theorem we can claim that
\begin{equation}\label{eq:ch2}
\mbs{\xi}^{\wdg^4} + \phi^2 \mbs{\xi}^{\wdg^2} = \mbf{0},
\end{equation}
since $\mbs{\xi}^\wdg$ must satisfy its own characteristic equation.

Finally, for adjoint poses we have that the characteristic equation of $\mbs{\xi}^\Wdg$ is
\begin{equation}\label{eq:adse3char}
\left|  \lambda \mbf{1} - \mbs{\xi}^\Wdg \right| = \left| \begin{matrix} \lambda \mbf{1} - \mbs{\phi}^\wdg & -\mbs{\rho}^\wdg \\ \mbf{0} & \lambda \mbf{1} - \mbs{\phi}^\wdg \end{matrix} \right| =  \left(  \lambda^3 + \phi^2 \lambda  \right)^2 =  \lambda^6 + 2 \phi^2 \lambda^4 + \phi^4 \lambda^2 = 0.
\end{equation}
We could again employ the Cayley-Hamilton theorem to create a similar identity to the other two cases, but it turns out that the {\em minimal polynomial} of $\mbs{\xi}^\Wdg$ is actually \citep{barfoot_ser24, deleuterio_rspa22},
\begin{equation}\label{eq:adid}
\mbs{\xi}^{\Wdg^5} + 2\phi^2 \mbs{\xi}^{\Wdg^3} + \phi^4 \mbs{\xi}^\Wdg = \mbf{0},
\end{equation}
which is one order lower than the characteristic equation in this case.  This is because both the algebraic and geometric multiplicities of eigenvalue $\lambda=0$ in~\eqref{eq:adse3char} are two (two Jordan blocks of size one), allowing us to drop one copy when constructing the minimal polynomial; in~\eqref{eq:se3char} the algebraic multiplicity of $\lambda = 0$ is two, while the geometric multiplicity is one (one Jordan block of size two).

\section{Building Blocks}\label{sec:bb}

This section introduces a collection of building-block functions and then discusses some connections between them.  The main contribution of the paper is to show that the substructures of many common matrix Lie groups can be constructed from these building blocks and moreover that we can employ an integral relationship to make some calculations quite efficient.  

Below are the building-block functions.  Each building block is actually a sequence indexed on $\ell=0,1,2,\ldots$:
\begin{subequations}\label{eq:bb}
\begin{gather}
\bg_\ell(x) = \sum_{m=0}^\infty \frac{1}{(\ell+m)!}x^m,  \;\; \bg_\ell(x, y) = \sum_{m=0}^\infty \frac{1}{(\ell + m + 1)!} x^m \, y,  \;\;
\bg_\ell(x, y, z) = \sum_{m=0}^\infty \sum_{n=0}^\infty \frac{1}{(\ell+m+n+1)!} x^m \, y \, z^n, \\  
\bg_\ell(x, \tau) = \sum_{m=0}^\infty \frac{m + 1}{(\ell + m + 1)!} x^m \, \tau, \quad \bg_\ell(x, y, \tau) = \sum_{m=0}^\infty \frac{1}{(\ell + m + 2)!} x^m \, y \, \tau, \\
\bg_\ell(x,y,z,\tau) = \sum_{m=0}^\infty \sum_{n=0}^\infty \frac{m+1}{(\ell+m+n+2)} x^m \,y \,z^n \, \tau .
\end{gather}
\end{subequations}
The first of these was presented previously by \citet{bloeschStateEstimationLegged2013,fornasierEquivariantSymmetriesAided2024}.  Notationally, we consider these functions to be templates where the variables $x, y, z$ are overloaded for all input types (i.e., scalar, column, matrix) that are compatible with the specific order of variables used (i.e., we do not a priori assume $x$, $y$, $z$ commute). The variable $\tau$ is always a scalar and will not be substituted.  Moreover, we will substitute $\bg$, $\Bg$, or $\BG$ for the function label depending on whether the output is a scalar, column, or matrix, respectively.   We will see that the first row of building blocks are the most common with the others required for one specific group, $SGal(3)$, to be discussed later in the paper.  Several example usages of our building-block templates are as follows:
\begin{gather}
\BG_\ell(\mbs{\phi}^\wdg) = \sum_{m=0}^\infty \frac{1}{(\ell+m)!}\mbs{\phi}^{\wdg^m}, \quad \Bg_\ell(\mbs{\phi}^\wdg, \mbs{\rho}) = \sum_{m=0}^\infty \frac{1}{(\ell + m + 1)!} \mbs{\phi}^{\wdg^m} \, \mbs{\rho},  \\
\BG_\ell(\mbs{\xi}^\Wdg) = \sum_{m=0}^\infty \frac{1}{(\ell+m)!}\mbs{\xi}^{\Wdg^m}, \quad \BG_\ell(\mbs{\phi}^\wdg, \mbs{\rho}^\wdg, \mbs{\phi}^\wdg) = \sum_{m=0}^\infty \sum_{n=0}^\infty \frac{1}{(\ell+m+n+1)!} \mbs{\phi}^{\wdg^m} \, \mbs{\rho}^\wdg \, \mbs{\phi}^{\wdg^n}, \\ \bg_\ell(-\lambda) = \sum_{m=0}^\infty \frac{1}{(\ell+m)!}(-\lambda)^m, \quad \BG_\ell(\mbs{\phi}^\wdg, \mbs{\rho}^\wdg, \mbs{\phi}^\wdg, \tau) = \sum_{m=0}^\infty \sum_{n=0}^\infty \frac{m+1}{(\ell+m+n+2)} \mbs{\phi}^{\wdg^m} \, \mbs{\rho}^\wdg \, \mbs{\phi}^{\wdg^n} \, \tau, \\
\Bg_\ell(\mbs{\phi}^\wdg, \mbs{\rho}, -\lambda) = \sum_{m=0}^\infty \sum_{n=0}^\infty \frac{1}{(\ell+m+n+1)!} \mbs{\phi}^{\wdg^m} \, \mbs{\rho} \, (-\lambda)^n.
\end{gather}
There is a remark worth making about connections between some of the building blocks:
\begin{remark}\label{thm:remark1}
The following relationships hold between building blocks in~\eqref{eq:bb}:
\begin{equation}
\bg_\ell(x,y) = \bg_{\ell+1}(x) \, y, \quad \bg_\ell(x,y,\tau) = \bg_{\ell+1}(x,y) \, \tau = \bg_{\ell + 2}(x) \, y \, \tau.
\end{equation}
\end{remark}

We can see right away some quantities that we introduced in the background section are equivalent to some of these building blocks.  For example, for $SO(3)$, we have
\begin{equation}
\mbf{C}(\mbs{\phi}) = \BG_0(\mbs{\phi}^\wdg), \quad \mbf{J}(\mbs{\phi}) = \BG_1(\mbs{\phi}^\wdg).
\end{equation}
Moreover, $\mbf{N}(\mbs{\phi}) = \BG_2(\mbs{\phi}^\wdg)$ shows up in inertial navigation equations \citep{bloeschStateEstimationLegged2013,fornasierEquivariantSymmetriesAided2024, barfoot_ser24}.  Later we will see that these building blocks are interrelated and can be used to define the substructure of several common matrix Lie groups.

Now, it is also fairly well known that there is an integral relationship between $\mbf{C}(\mbs{\phi})$ and $\mbf{J}(\mbs{\phi})$,
\begin{equation}
\mbf{J}(\mbs{\phi}) = \int_0^1 \mbf{C}(\alpha \mbs{\phi}) \, d\alpha,
\end{equation}
which dates back at least to \citet[p.21]{parkOptimalKinematicDesign1991} and can be found in later works as well \citep{barfoot_tro14, barfoot_ser24}. Our first contribution of the current work is to notice that this integral relationship can be generalized to our entire collection of building blocks.

\begin{lemma}\label{thm:bbrecursive}
Given the building-block functions in~\eqref{eq:bb}, we have the following recursive integral relationships:
\begin{subequations}
\begin{gather}
\bg_{\ell+1}(x) = \int_0^1 \alpha^\ell \, \bg_\ell(\alpha x) \, d\alpha, \quad \bg_{\ell+1}(x, y) = \int_0^1 \alpha^\ell \bg_\ell(\alpha x, \alpha y) \, d\alpha, \;\; \bg_{\ell+1}(x,y,z) = \int_0^1 \alpha^\ell \bg_\ell(\alpha x,\alpha y, \alpha z) \, d\alpha, \\
\bg_{\ell+1}(x,\tau) = \int_0^1 \alpha^\ell \bg_\ell(\alpha x, \alpha \tau) \, d\alpha, \quad \bg_{\ell+1}(x,y,\tau) = \int_0^1 \alpha^\ell \bg_\ell(\alpha x, \alpha y, \alpha \tau) \, d\alpha, \\
\bg_{\ell+1}(x, y, z, \tau) = \int_0^1 \alpha^\ell \bg_\ell(\alpha x,\alpha y, \alpha z, \alpha \tau) \, d\alpha.
\end{gather}
\end{subequations}
\end{lemma}
\begin{proof}  We will show a few examples and leave the rest to the reader.
For $\bg_\ell(x)$ we have
\begin{multline}
\int_0^1 \alpha^\ell \,\bg_\ell(\alpha x) \, d\alpha = \int_0^1 \alpha^\ell \, \left( \sum_{m=0}^\infty \frac{1}{(\ell+m)!}(\alpha x)^m\right) \, d\alpha = \sum_{m=0}^\infty \frac{\int_0^1 \alpha^{\ell+m} \, d\alpha}{(\ell+m)!} x^m \\ = \sum_{m=0}^\infty \frac{1}{(\ell+m+1)!} x^m = \bg_{\ell+1}(x).
\end{multline}
Similarly, for $\bg_\ell(x,y,z)$ we have
\begin{multline}
\int_0^1 \alpha^\ell \, \bg_\ell(\alpha x, \alpha y, \alpha z) \, d\alpha = \int_0^1 \alpha^\ell \, \left(\sum_{m=0}^\infty \sum_{n=0}^\infty \frac{1}{(\ell+m+n+1)!} (\alpha x)^m (\alpha y) (\alpha z)^n\right) \, d\alpha \\ = \sum_{m=0}^\infty \sum_{n=0}^\infty \frac{\int_0^1 \alpha^{\ell+m+n+1} \, d\alpha}{(\ell+m+n+1)!} x^m \, y \, z^n  = \sum_{m=0}^\infty\sum_{n=0}^\infty \frac{1}{(\ell+m+n+2)!} x^m \, y \, z^n = \bg_{\ell+1}(x, y, z).
\end{multline}
The rest of the collection follows similar reasoning.
\end{proof}

The immediate application of Lemma~\ref{thm:bbrecursive} is that once we have a compact $\ell=0$ expression, we can simply integrate for higher values of $\ell$.   For example, for  $\BG_0(\mbs{\phi}^\wdg) = \mbf{C}(\mbs{\phi})$ provided in~\eqref{eq:Ccompact}, we can integrate to produce the higher-order members of the sequence.  To compute $\BG_1(\mbs{\phi}^\wdg) = \mbf{J}(\mbs{\phi})$ we have
\begin{multline}
\BG_1(\mbs{\phi}^\wdg) = \int_0^1 \alpha^0 \, \BG_0(\alpha \mbs{\phi}^\wdg) \, d\alpha =  \int_0^1 \left( c_0(\alpha\phi) \mbf{1} +  \alpha c_1(\alpha\phi) \mbs{\phi}^\wdg + \alpha^2 c_2(\alpha\phi) \mbs{\phi}^{\wdg^2} \right) d\alpha  \\ = \left(\int_0^1 c_0(\alpha\phi) \, d\alpha \right) \mbf{1} + \left(\int_0^1  \alpha c_1(\alpha\phi) \, d\alpha\right) \mbs{\phi} + \left(\int_0^1 \alpha^2 c_2(\alpha\phi) \, d\alpha\right)  \mbs{\phi}^{\wdg^2} = j_0(\phi) \mbf{1} + j_1(\phi) \mbs{\phi}^\wdg + j_2(\phi) \mbs{\phi}^{\wdg^2},
\end{multline}
which matches~\eqref{eq:Jcompact}.  The advantage here is that we do not need to reintroduce the minimal polynomial to make our expression compact. Instead, only three scalar integrals need be computed, possibly exploiting symbolic math software.  We can continue the process to compute any desired $\BG_\ell(\mbs{\phi}^\wdg)$.  This process is depicted in~\eqref{eq:cd} (solid arrow path).

As another example, consider $\mbox{Ad}(SE(3))$.  Using the minimal polynomial in~\eqref{eq:adid}, it can be shown that \citep[(8.68)]{barfoot_ser24}
\begin{multline}\label{eq:SE3_monoT}
\Tbig(\mbs{\xi}) = \BG_0(\mbs{\xi}^\Wdg) = \underbrace{1}_{t_0(\phi)} \mbf{1} + \underbrace{\left( \frac{3\sin\phi - \phi\cos\phi}{2\phi} \right)}_{t_1(\phi)} \mbs{\xi}^\Wdg + \underbrace{\left( \frac{4-\phi\sin\phi-4\cos\phi}{2\phi^2} \right)}_{t_2(\phi)} \mbs{\xi}^{\Wdg^2} \\
+ \underbrace{\left( \frac{\sin\phi-\phi\cos\phi}{2\phi^3} \right)}_{t_3(\phi)} \mbs{\xi}^{\Wdg^3} + \underbrace{\left( \frac{2-\phi\sin\phi-2\cos\phi}{2\phi^4} \right)}_{t_4(\phi)} \mbs{\xi}^{\Wdg^4}.
\end{multline}
Then, using our recursive relationship, $\BG_1(\mbs{\xi}^\Wdg) = \int_0^1 \BG_0(\alpha \mbs{\xi}^\Wdg) \, d\alpha$, we have
\begin{multline}\label{eq:SE3_monoJ}
\mbs{\mathcal{J}}(\mbs{\xi}) = \BG_1(\mbs{\xi}^\Wdg) = \sum_{m=0}^4 \left( \int_0^1 \alpha^m t_m(\alpha\phi) \, d\alpha \right) \mbs{\xi}^{\Wdg^m} 
\\ = \mbf{1} + \left( \frac{4-\phi\sin\phi - 4\cos\phi}{2\phi^2} \right) \mbs{\xi}^\Wdg + \left( \frac{4\phi - 5\sin\phi+\phi\cos\phi}{2\phi^3} \right) \mbs{\xi}^{\Wdg^2} \\ + \left( \frac{2 - \phi\sin\phi-2\cos\phi}{2\phi^4} \right) \mbs{\xi}^{\Wdg^3} + \left( \frac{2\phi-3\sin\phi+\phi\cos\phi}{2\phi^5} \right) \mbs{\xi}^{\Wdg^4},
\end{multline}
which matches the result of \citet[(8.99)]{barfoot_ser24} derived by a different means (i.e., expanding the infinite series and then applying the minimal polynomial).  In this case, we were required to compute five scalar integrals because the minimal polynomial is of higher order.  Nevertheless, our approach is simpler and hence less error prone than applying the minimal polynomial to the infinite series and then attempting to recognize the trigonometric coefficients in the resulting finite-series forms.

\section{Substructures}\label{sec:substructure}

This section studies several matrix Lie groups, making use of our building blocks to reveal important substructure and additional relationships.


\subsection{Poses, $SE(3)$}

We will investigate the substructures of $SE(3)$ and $\mbox{Ad}(SE(3))$ more closely.  The following theorem relates the building blocks of $SE(3)$ to those of $SO(3)$:

\begin{theorem}\label{thm:bbse3so3}
$SE(3)$:  Given the building-block functions in~\eqref{eq:bb}, we have that
\begin{equation}
\BG_\ell(\mbs{\xi}^\wdg) = \bbm \BG_\ell(\mbs{\phi}^\wdg) & \Bg_\ell(\mbs{\phi}^\wdg, \mbs{\rho}) \\ \mbf{0}^T & \frac{1}{\ell !} \ebm.
\end{equation}
In other words, for all $\ell$ the $SE(3)$ building blocks can be constructed from the $SO(3)$ building blocks.
\end{theorem}
\begin{proof}
We will use a proof by induction.  For the $\ell=0$ case we have to show that
\begin{equation}
\BG_0(\mbs{\xi}^\wdg) = \bbm \BG_0(\mbs{\phi}^\wdg) & \Bg_0(\mbs{\phi}^\wdg, \mbs{\rho}) \\ \mbf{0}^T & 1 \ebm = \bbm \mbf{C}(\mbs{\phi}) & \mbf{J}(\mbs{\phi}) \mbs{\rho}  \\ \mbf{0}^T & 1 \ebm = \Tsmall(\mbs{\xi}).
\end{equation}
By definition $\BG_0(\mbs{\phi}^\wdg) = \mbf{C}(\mbs{\phi})$ and by Remark~\ref{thm:remark1} we have $\Bg_0(\mbs{\phi}^\wdg, \mbs{\rho}) = \BG_1(\mbs{\phi}^\wdg) \mbs{\rho} = \mbf{J}(\mbs{\phi}) \mbs{\rho}$ so this holds.  We then assume the relationship holds for $\ell$ and show this implies it holds for $\ell+1$:
\begin{multline}
\BG_{\ell+1}(\mbs{\xi}^\wdg) = \int_0^1 \alpha^\ell \, \BG_\ell(\alpha \mbs{\xi}^\wdg) \, d\alpha = \int_0^1 \alpha^\ell \bbm \BG_\ell(\alpha\mbs{\phi}^\wdg) & \Bg_\ell(\alpha\mbs{\phi}^\wdg, \alpha\mbs{\rho}) \\ \mbf{0}^T & \frac{1}{\ell !}  \ebm \, d\alpha \\ = \bbm \int_0^1 \alpha^\ell \BG_\ell(\alpha\mbs{\phi}^\wdg) \, d\alpha & \int_0^1 \alpha^\ell \Bg_\ell(\alpha\mbs{\phi}^\wdg, \alpha\mbs{\rho}) \, d\alpha \\ \mbf{0}^T & \frac{\int_0^1 \alpha^\ell \, d\alpha}{\ell !}  \ebm = \bbm \BG_{\ell+1}(\mbs{\phi}^\wdg) & \Bg_{\ell+1}(\mbs{\phi}^\wdg, \mbs{\rho}) \\ \mbf{0}^T & \frac{1}{(\ell + 1)!} \ebm,
\end{multline}
where we employ Lemma~\ref{thm:bbrecursive} in the first and last steps.
\end{proof}

A similar result can be shown relating the building blocks of $\mbox{Ad}(SE(3))$ to those of $SO(3)$:

\begin{theorem}\label{thm:bbadse3so3}
$\mbox{\normalfont Ad}(SE(3))$: Given the building-block functions in~\eqref{eq:bb}, we have that
\begin{equation}
\BG_\ell(\mbs{\xi}^\Wdg) = \bbm \BG_\ell(\mbs{\phi}^\wdg) & \BG_\ell(\mbs{\phi}^\wdg, \mbs{\rho}^\wdg, \mbs{\phi}^\wdg) \\ \mbf{0} & \BG_\ell(\mbs{\phi}^\wdg) \ebm.
\end{equation}
In other words, for all $\ell$ the $\mbox{\normalfont Ad}(SE(3))$ building blocks can be constructed from the $SO(3)$ building blocks.
\end{theorem}
\begin{proof}
We will again use a proof by induction.  For the $\ell=0$ case we have to show that
\begin{equation}
\BG_0(\mbs{\xi}^\Wdg) = \bbm \BG_0(\mbs{\phi}^\wdg) & \BG_0(\mbs{\phi}^\wdg, \mbs{\rho}^\wdg, \mbs{\phi}^\wdg) \\ \mbf{0} & \BG_0(\mbs{\phi}^\wdg) \ebm = \bbm \mbf{C}(\mbs{\phi}) & \left( \mbf{J}(\mbs{\phi}) \mbs{\rho} \right)^\wdg \mbf{C}(\mbs{\phi}) \\ \mbf{0} & \mbf{C}(\mbs{\phi}) \ebm = \Tbig(\mbs{\xi}).
\end{equation}
By definition $\BG_0(\mbs{\phi}^\wdg) = \mbf{C}(\mbs{\phi})$ and from Lemma~\ref{thm:bbinit1} we have
\begin{equation}
\BG_0(\mbs{\phi}^\wdg, \mbs{\rho}^\wdg, \mbs{\phi}^\wdg) =  \Bg_0(\mbs{\phi}^\Wdg,\mbs{\rho})^\wdg \BG_0(\mbs{\phi}^\wdg) =  \left(\BG_1(\mbs{\phi}^\Wdg)\mbs{\rho})\right)^\wdg \BG_0(\mbs{\phi}^\wdg) =  (\left( \mbf{J}(\mbs{\phi}) \mbs{\rho} \right)^\wdg \mbf{C}(\mbs{\phi}),
\end{equation}
where we note that $\wdg = \Wdg$ for $SO(3)$ as it is self-adjoint.
Next, we assume that the result holds for $\ell$ and show that this implies it holds for $\ell+1$.  Using Lemma~\ref{thm:bbrecursive} we can see that
\begin{multline}
\BG_{\ell+1}(\mbs{\xi}^\Wdg) = \int_0^1 \alpha^\ell \, \BG_\ell(\alpha \mbs{\xi}^\Wdg) \, d\alpha = \int_0^1 \alpha^\ell \bbm \BG_\ell(\alpha\mbs{\phi}^\wdg) & \BG_\ell(\alpha\mbs{\phi}^\wdg, \alpha\mbs{\rho}^\wdg, \alpha\mbs{\phi}^\wdg) \\ \mbf{0} & \BG_\ell(\alpha\mbs{\phi}^\wdg) \ebm \, d\alpha \\ = \bbm \int_0^1 \alpha^\ell \BG_\ell(\alpha\mbs{\phi}^\wdg) \, d\alpha & \int_0^1 \alpha^\ell \BG_\ell(\alpha\mbs{\phi}^\wdg, \alpha\mbs{\rho}^\wdg, \alpha\mbs{\phi}^\wdg) \, d\alpha \\ \mbf{0} & \int_0^1 \alpha^\ell \BG_\ell(\alpha\mbs{\phi}^\wdg) \, d\alpha \ebm = \bbm \BG_{\ell+1}(\mbs{\phi}^\wdg) & \BG_{\ell+1}(\mbs{\phi}^\wdg, \mbs{\rho}^\wdg, \mbs{\phi}^\wdg) \\ \mbf{0} & \BG_{\ell+1}(\mbs{\phi}^\wdg) \ebm.
\end{multline}
Therefore the result holds for all $\ell$.
\end{proof}

The implication of Theorem~\ref{thm:bbadse3so3} is that we can now compute the blocks of $\BG_\ell(\mbs{\xi}^\Wdg)$ more easily.  For example, we know that
\begin{equation}
\BG_1(\mbs{\xi}^\Wdg) = \bbm \BG_1(\mbs{\phi}^\wdg) & \BG_1(\mbs{\phi}^\wdg, \mbs{\rho}^\wdg, \mbs{\phi}^\wdg) \\ \mbf{0} & \BG_1(\mbs{\phi}^\wdg) \ebm = \bbm \mbf{J}(\mbs{\phi}) & \mbf{Q}(\mbs{\phi}, \mbs{\rho}) \\ \mbf{0} & \mbf{J}(\mbs{\phi}) \ebm = \mbs{\mathcal{J}}(\mbs{\xi}).
\end{equation}
Comparing blocks we see that
\begin{equation}
\mbf{Q}(\mbs{\phi}, \mbs{\rho}) = \BG_1(\mbs{\phi}^\wdg, \mbs{\rho}^\wdg, \mbs{\phi}^\wdg) = \int_0^1 \BG_0(\alpha\mbs{\phi}^\wdg, \alpha\mbs{\rho}^\wdg, \alpha\mbs{\phi}^\wdg) \,d\alpha 
\end{equation}
is an integral form for the upper-right block of the $SE(3)$ Jacobian.  It will be easier if we first compute $\BG_0(\mbs{\phi}^\wdg, \mbs{\rho}^\wdg, \mbs{\phi}^\wdg)$ and then integrate:
\begin{equation}
\BG_0(\mbs{\phi}^\wdg, \mbs{\rho}^\wdg, \mbs{\phi}^\wdg) = \left( \mbf{J}(\mbs{\phi}) \mbs{\rho} \right)^\wdg \mbf{C}(\mbs{\phi}) = \left( \left( j_0(\phi) \mbf{1} + j_1(\phi) \mbs{\phi}^\wdg + j_2(\phi) \mbs{\phi}^{\wdg^2} \right) \mbs{\rho} \right)^\wdg \left( c_0(\phi) \mbf{1} + c_1(\phi) \mbs{\phi}^\wdg + c_2(\phi) \mbs{\phi}^{\wdg^2}  \right) .
\end{equation}
We can apply Lemma~\ref{thm:product} in Appendix~\ref{sec:identities} to expand this product:
\begin{multline}
\BG_0(\mbs{\phi}^\wdg, \mbs{\rho}^\wdg, \mbs{\phi}^\wdg) = \mbs{\rho}^\wdg + j_1(\phi) \mbs{\phi}^\wdg \mbs{\rho}^\wdg + \underbrace{\left(c_1(\phi) - j_1(\phi) + \phi^2 (j_1(\phi)c_2(\phi) - j_2(\phi) c_1(\phi) \right)}_{j_1(\phi)} \mbs{\rho}^\wdg \mbs{\phi}^\wdg   \\ + j_2(\phi) \mbs{\phi}^{\wdg^2} \mbs{\rho}
+ \left(  j_1(\phi)c_1(\phi) - 2 j_2(\phi) + \phi^2  j_2(\phi)c_2(\phi) \right) \mbs{\phi}^\wdg \mbs{\rho}^\wdg \mbs{\phi}^\wdg + j_2(\phi)c_1(\phi)  \mbs{\phi}^{\wdg^2} \mbs{\rho}^\wdg \mbs{\phi}^\wdg \qquad\qquad\quad \\ + \underbrace{\left( c_2(\phi) + j_2(\phi) - j_1(\phi)c_1(\phi) - \phi^2 j_2(\phi)c_2(\phi) \right)}_{j_2(\phi)} \mbs{\rho}^\wdg \mbs{\phi}^{\wdg^2} 
 + \left( j_1(\phi)c_2(\phi) - 2 j_2(\phi)c_1(\phi)\right)   \mbs{\phi}^\wdg \mbs{\rho}^\wdg \mbs{\phi}^{\wdg^2} \\
= \mbs{\rho}^\wdg + j_1(\phi) \left( \mbs{\phi}^\wdg \mbs{\rho}^\wdg +  \mbs{\rho}^\wdg \mbs{\phi}^\wdg + \mbs{\phi}^\wdg \mbs{\rho}^\wdg \mbs{\phi}^\wdg \right) + j_2(\phi) \left(  \mbs{\phi}^{\wdg^2} \mbs{\rho}^\wdg + \mbs{\rho}^\wdg \mbs{\phi}^{\wdg^2} - 2  \mbs{\phi}^\wdg \mbs{\rho}^\wdg \mbs{\phi}^\wdg \right) \qquad \\ \frac{1}{2} \left( j_1(\phi)c_2(\phi) - j_2(\phi)c_1(\phi)\right) \left( \mbs{\phi}^{\wdg^2} \mbs{\rho}^\wdg \mbs{\phi}^\wdg +  \mbs{\phi}^\wdg \mbs{\rho}^\wdg \mbs{\phi}^{\wdg^2} \right).
\end{multline}
We are left now with only computing five simple integrals to arrive at $\BG_1(\mbs{\phi}^\wdg, \mbs{\rho}^\wdg, \mbs{\phi}^\wdg)$:
\begin{multline}\label{eq:SE3_Q}
\mbf{Q}(\mbs{\phi}, \mbs{\rho}) = \BG_1(\mbs{\phi}^\wdg, \mbs{\rho}^\wdg, \mbs{\phi}^\wdg) = \int_0^1 \BG_0(\alpha\mbs{\phi}^\wdg, \alpha\mbs{\rho}^\wdg, \alpha\mbs{\phi}^\wdg) \,d\alpha \\
= \left( \int_0^1 \alpha \, d\alpha \right) \mbs{\rho}^\wdg + \left( \int_0^1 \alpha^2 j_1(\alpha\phi) \, d\alpha \right) \left( \mbs{\phi}^\wdg \mbs{\rho}^\wdg +  \mbs{\rho}^\wdg \mbs{\phi}^\wdg \right) + \left( \int_0^1 \alpha^3 j_1(\alpha\phi) \, d\alpha \right) \mbs{\phi}^\wdg \mbs{\rho}^\wdg \mbs{\phi}^\wdg  \\
+ \left( \int_0^1 \alpha^3 j_2(\alpha\phi) \, d\alpha \right) \left(  \mbs{\phi}^{\wdg^2} \mbs{\rho}^\wdg + \mbs{\rho}^\wdg \mbs{\phi}^{\wdg^2} - 2  \mbs{\phi}^\wdg \mbs{\rho}^\wdg \mbs{\phi}^\wdg \right) \hspace*{1.4in} \\ 
+ \frac{1}{2} \left( \int_0^1 \alpha^4 \left( j_1(\alpha\phi)c_2(\alpha\phi) - j_2(\alpha\phi)c_1(\alpha\phi)\right) \, d\alpha \right) \left( \mbs{\phi}^{\wdg^2} \mbs{\rho}^\wdg \mbs{\phi}^\wdg +  \mbs{\phi}^\wdg \mbs{\rho}^\wdg \mbs{\phi}^{\wdg^2} \right)  \\
 =  \; \frac{1}{2} \mbs{\rho}^\wdg + \left(\frac{\phi-\sin\phi}{\phi^3}\right) \left(  \mbs{\phi}^\wdg \mbs{\rho}^\wdg + \mbs{\rho}^\wdg \mbs{\phi}^\wdg + \mbs{\phi}^\wdg \mbs{\rho}^\wdg \mbs{\phi}^\wdg \right) \hspace*{2.2in} \\
 + \left(\frac{\phi^2+2\cos\phi-2}{2\phi^4}\right) \left( \mbs{\phi}^{\wdg^2} \mbs{\rho}^\wdg + \mbs{\rho}^\wdg \mbs{\phi}^{\wdg^2} - 3 \mbs{\phi}^\wdg \mbs{\rho}^\wdg \mbs{\phi}^\wdg \right)  \\
 + \left( \frac{2\phi-3\sin\phi+\phi\cos\phi}{2\phi^5}\right) \left(  \mbs{\phi}^{\wdg^2} \mbs{\rho}^\wdg \mbs{\phi}^\wdg + \mbs{\phi}^\wdg \mbs{\rho}^\wdg \mbs{\phi}^{\wdg^2} \right), 
\end{multline}
which matches the result of \citet{barfoot_tro14} and has later appeared in \citet{sola18, kellyAllGalileanGroup2024a,fornasierEquivariantSymmetriesAided2024, barfoot_ser24}.  While this may still seem onerous, it is less taxing than working with the infinite-series expressions due to the manipulations being limited to a fixed number of terms and the option to use a symbolic math package to evaluate the integrals.  We could continue to integrate to compute any required $\BG_\ell(\mbs{\phi}^\wdg, \mbs{\rho}^\wdg, \mbs{\phi}^\wdg)$.


\subsection{Extended Poses, $SE_2(3)$}

The group $SE_2(3)$ provides an elegant means to pose inertial navigation problems \citep{barrau15, barrau18, brossard21, barfoot_ser24}.  The main idea is to extend $SE(3)$ to also contain the translational velocity, resulting in a $5 \times 5$ {\em extended-pose} matrix:
\begin{equation}
\Tsmall(\mbs{\xi}) = \exp\left( \mbs{\xi}^\wdg \right) = \exp \left( \bbm \mbs{\rho}  \\ \mbs{\nu} \\ \mbs{\phi} \ebm^\wdg \right) = \exp \left( \bbm \mbs{\phi}^\wdg & \mbs{\nu} & \mbs{\rho} \\ \mbf{0}^T & 0 & 0 \\ \mbf{0}^T & 0 & 0 \ebm \right) = \bbm \mbf{C}(\mbs{\phi}) & \mbf{J}(\mbs{\phi}) \mbs{\nu} & \mbf{J}(\mbs{\phi}) \mbs{\rho} \\ \mbf{0}^T & 1 & 0 \\ \mbf{0}^T & 0 & 1 \ebm,
\end{equation}
where $\mbf{J}(\mbs{\phi})$ is the left Jacobian of $SO(3)$ and we have extended the $\wdg$ operator appropriately.  Naturally, the mapping is again surjective only.  The $9 \times 9$ adjoint for $SE_2(3)$ is given by
\begin{equation}
\Tbig(\mbs{\xi}) = \exp\left( \mbs{\xi}^\Wdg \right) = \exp\left( \bbm \mbs{\rho} \\ \mbs{\nu} \\ \mbs{\phi}  \ebm^\Wdg \right) = \exp\left( \bbm \mbs{\phi}^\wdg  & \mbf{0} & \mbs{\rho}^\wdg \\ \mbf{0} & \mbs{\phi}^\wdg & \mbs{\nu}^\wdg \\  \mbf{0} & \mbf{0} & \mbs{\phi}^\wdg \ebm \right)  = \bbm \mbf{C}(\mbs{\phi})  & \mbf{0} & (\mbf{J}(\mbs{\phi})\mbs{\rho})^\wdg \mbf{C}(\mbs{\phi}) \\ \mbf{0} & \mbf{C}(\mbs{\phi}) & (\mbf{J}(\mbs{\phi})\mbs{\nu})^\wdg  \mbf{C}(\mbs{\phi}) \\ \mbf{0} & \mbf{0} & \mbf{C}(\mbs{\phi}) \ebm,
\end{equation}
where we have extended the adjoint $\Wdg$ operator appropriately.  
The left Jacobian of $SE_2(3)$ is given by
\begin{equation}\label{eq:SE23Jac}
\Jbig(\mbs{\xi}) = \bbm \mbf{J}(\mbs{\phi})  & \mbf{0} & \mbf{Q}(\mbs{\phi},\mbs{\rho}) \\ \mbf{0} & \mbf{J}(\mbs{\phi}) & \mbf{Q}(\mbs{\phi},\mbs{\nu})  \\ \mbf{0} & \mbf{0} & \mbf{J}(\mbs{\phi}) \ebm,
\end{equation}
where $\mbf{Q}(\mbs{\phi},\cdot)$ is provided in~\eqref{eq:Q}.  

Given the similarity to $SE(3)$, the following corollaries to Theorems~\ref{thm:bbse3so3} and~\ref{thm:bbadse3so3} should not be a surprise:

\begin{corollary}\label{thm:bbse23so3}
$SE_2(3)$: Following on to Theorem~\ref{thm:bbse3so3}, given the building-block functions in~\eqref{eq:bb} we have that
\begin{equation}
\BG_\ell(\mbs{\xi}^\wdg) = \bbm \BG_\ell(\mbs{\phi}^\wdg) & \Bg_\ell(\mbs{\phi}^\wdg, \mbs{\nu}) & \Bg_\ell(\mbs{\phi}^\wdg, \mbs{\rho}) \\ \mbf{0}^T & \frac{1}{\ell!} & 0 \\ \mbf{0}^T & 0 & \frac{1}{\ell!} \ebm.
\end{equation}
In other words, for all $\ell$ the $SE_2(3)$ building blocks can be constructed from the $SO(3)$ building blocks.
\end{corollary}

\begin{corollary}\label{thm:bbadse23so3}
$\mbox{\normalfont Ad}(SE_2(3))$: Following on to Theorem~\ref{thm:bbadse3so3}, given the building-block functions in~\eqref{eq:bb} we have that
\begin{equation}
\BG_\ell(\mbs{\xi}^\Wdg) = \bbm \BG_\ell(\mbs{\phi}^\wdg) & \mbf{0} & \BG_\ell(\mbs{\phi}^\wdg, \mbs{\rho}^\wdg,\mbs{\phi}^\wdg) \\ \mbf{0} & \BG_\ell(\mbs{\phi}^\wdg) & \BG_\ell(\mbs{\phi}^\wdg, \mbs{\nu}^\wdg,\mbs{\phi}^\wdg) \\ \mbf{0} & \mbf{0} & \BG_\ell(\mbs{\phi}^\wdg) \ebm.
\end{equation}
In other words, for all $\ell$ the $\mbox{\normalfont Ad}(SE_2(3))$ building blocks can be constructed from the $SO(3)$ building blocks.
\end{corollary}
We omit the proofs due to their similarity to the ones for Theorems~\ref{thm:bbse3so3} and~\ref{thm:bbadse3so3}.


\subsection{Galilean Group, $SGal(3)$}

In a sense, the special Galilean group, $SGal(3)$, further extends $SE_2(3)$ by incorporating a scalar time, $\tau$, in the pose \citep{giefer21, kellyAllGalileanGroup2024a, kellyMakingSpaceTime2024, fornasierEquivariantSymmetriesAided2024, delamaEquivariantIMUPreintegration2024}.  This group can also be used in inertial navigation problems and can be used to represent both spatial and temporal uncertainty.  

We rely on \citet{kellyAllGalileanGroup2024a} for the background.  As with $SE_2(3)$, we have a $5 \times 5$ pose matrix:
\begin{equation}
\mbf{T}(\mbs{\xi}) = \exp\left( \mbs{\xi}^\wdg \right) = \exp \left( \bbm \mbs{\rho}  \\ \mbs{\nu} \\ \mbs{\phi} \\ \tau \ebm^\wdg \right) = \exp \left( \bbm \mbs{\phi}^\wdg & \mbs{\nu} & \mbs{\rho} \\ \mbf{0}^T & 0 & \tau \\ \mbf{0}^T & 0 & 0 \ebm \right) = \bbm \mbf{C}(\mbs{\phi}) & \mbf{J}(\mbs{\phi}) \mbs{\nu} & \mbf{J}(\mbs{\phi}) \mbs{\rho} +   \mbf{N}(\mbs{\phi}) \mbs{\nu} \, \tau \\ \mbf{0}^T & 1 & \tau \\ \mbf{0}^T & 0 & 1 \ebm,
\end{equation}
where $\mbf{N}(\mbs{\phi}) = \mbs{\Gamma}_2(\mbs{\phi})$ and we have again implicitly defined a new version of $\wdg$ for $SGal(3)$.  The $10 \times 10$ adjoint for $SE_2(3)$ is given by
\begin{multline}
\Tbig(\mbs{\xi}) = \exp\left( \mbs{\xi}^\Wdg \right) = \exp\left( \bbm \mbs{\rho} \\ \mbs{\nu} \\ \mbs{\phi} \\ \tau \ebm^\Wdg \right) = \exp\left( \bbm \mbs{\phi}^\wdg  & -\mbf{1} \tau & \mbs{\rho}^\wdg & \mbs{\nu}^\wdg \\ \mbf{0} & \mbs{\phi}^\wdg & \mbs{\nu}^\wdg & \mbf{0} \\  \mbf{0} & \mbf{0} & \mbs{\phi}^\wdg & \mbf{0} \\ \mbf{0}^T & \mbf{0}^T & \mbf{0}^T & 0 \ebm \right) \\ 
= \bbm \mbf{C}(\mbs{\phi})  & -\mbf{C}(\mbs{\phi})\tau & \left(\mbf{J}(\mbs{\phi}) \mbs{\rho} -  (\mbf{J}(\mbs{\phi}) - \mbf{N}(\mbs{\phi}) )\mbs{\nu} \tau \right)^\wdg \mbf{C}(\mbs{\phi}) & \mbf{J}(\mbs{\phi}) \mbs{\nu} \\ \mbf{0} & \mbf{C}(\mbs{\phi}) & (\mbf{J}(\mbs{\phi})\mbs{\nu})^\wdg  \mbf{C}(\mbs{\phi}) & \mbf{0}  \\ \mbf{0} & \mbf{0} & \mbf{C}(\mbs{\phi}) & \mbf{0} \\ \mbf{0}^T & \mbf{0}^T & \mbf{0}^T & 1 \ebm,
\end{multline}
where we have extended the adjoint $\Wdg$ operator appropriately.  

With this setup, we can now state two more of our contributions:
\begin{theorem}\label{thm:bbgal3so3}
$SGal(3)$: Given the building-block functions in~\eqref{eq:bb}, we have that
\begin{equation}
\BG_\ell(\mbs{\xi}^\wdg) = \bbm \BG_\ell(\mbs{\phi}^\wdg) & \Bg_\ell(\mbs{\phi}^\wdg, \mbs{\nu}) &\Bg_\ell(\mbs{\phi}^\wdg, \mbs{\rho}) + \Bg_\ell(\mbs{\phi}^\wdg, \mbs{\nu}, \tau) \\ \mbf{0}^T & \frac{1}{\ell!} & \frac{\tau}{(\ell+1)!} \\ \mbf{0}^T & 0 & \frac{1}{\ell !} \ebm.
\end{equation}
In other words, for all $\ell$ the $SGal(3)$ building blocks can be constructed from the $SO(3)$ building blocks.
\end{theorem}
We omit the proof as it is similar to the one for the next theorem.
\smallskip

\begin{theorem}\label{thm:bbadgal3so3}
$\mbox{\normalfont Ad}(SGal(3))$: Given the building-block functions in~\eqref{eq:bb}, we have that
\begin{equation}
\BG_\ell(\mbs{\xi}^\Wdg) = \bbm \BG_\ell(\mbs{\phi}^\wdg) & -\BG_\ell(\mbs{\phi}^\wdg, \tau) & \BG_\ell(\mbs{\phi}^\wdg, \mbs{\rho}^\wdg, \mbs{\phi}^\wdg) - \BG_\ell(\mbs{\phi}^\wdg, \mbs{\nu}^\wdg, \mbs{\phi}^\wdg, \tau) & \Bg_\ell(\mbs{\phi}^\wdg, \mbs{\nu}) \\ \mbf{0} & \BG_\ell(\mbs{\phi}^\wdg) & \BG_\ell(\mbs{\phi}^\wdg, \mbs{\nu}^\wdg, \mbs{\phi}^\wdg) & \mbf{0} \\ \mbf{0} & \mbf{0} &  \BG_\ell(\mbs{\phi}^\wdg) & \mbf{0} \\ \mbf{0}^T & \mbf{0}^T & \mbf{0}^T & \frac{1}{\ell !} \ebm.
\end{equation}
In other words, for all $\ell$ the $\mbox{\normalfont Ad}(SGal(3))$ building blocks can be constructed from the $SO(3)$ building blocks.
\end{theorem}
\begin{proof}
Similarly to the previous theorems, we will use a proof by induction.  For the $\ell=0$ case we must have that
\begin{multline}
\BG_0(\mbs{\xi}^\Wdg) = \bbm \BG_0(\mbs{\phi}^\wdg) & -\BG_0(\mbs{\phi}^\wdg, \tau) & \BG_0(\mbs{\phi}^\wdg, \mbs{\rho}^\wdg, \mbs{\phi}^\wdg) - \BG_0(\mbs{\phi}^\wdg, \mbs{\rho}^\wdg, \mbs{\phi}^\wdg, \tau) & \Bg_0(\mbs{\phi}^\wdg, \mbs{\nu}) \\ \mbf{0} & \BG_0(\mbs{\phi}^\wdg) & \BG_0(\mbs{\phi}^\wdg, \mbs{\nu}^\wdg, \mbs{\phi}^\wdg) & \mbf{0} \\ \mbf{0} & \mbf{0} &  \BG_0(\mbs{\phi}^\wdg) & \mbf{0} \\ \mbf{0}^T & \mbf{0}^T & \mbf{0}^T & 1 \ebm \\ = \bbm \mbf{C}(\mbs{\phi})  & -\mbf{C}(\mbs{\phi})\tau & \left(\mbf{J}(\mbs{\phi}) \mbs{\rho} -  (\mbf{J}(\mbs{\phi}) - \mbf{N}(\mbs{\phi}) )\mbs{\nu} \tau \right)^\wdg \mbf{C}(\mbs{\phi}) & \mbf{J}(\mbs{\phi}) \mbs{\nu} \\ \mbf{0} & \mbf{C}(\mbs{\phi}) & (\mbf{J}(\mbs{\phi})\mbs{\nu})^\wdg  \mbf{C}(\mbs{\phi}) & \mbf{0}  \\ \mbf{0} & \mbf{0} & \mbf{C}(\mbs{\phi}) & \mbf{0} \\ \mbf{0}^T & \mbf{0}^T & \mbf{0}^T & 1 \ebm = \Tbig(\mbs{\xi}).
\end{multline}
Many of the blocks follow from arguments made in previous theorems, so will restrict our explanations to the new cases only.  For the $(1,4)$ block we have
$\Bg_0(\mbs{\phi}^\wdg, \mbs{\nu}) = \BG_{1}(\mbs{\phi}^\wdg) \mbs{\nu} = \mbf{J}(\mbs{\phi}) \mbs{\nu}$ and for the $(1,2)$ block we have
\begin{equation}
\BG_0(\mbs{\phi}^\wdg, \tau)  =  \sum_{m=0}^\infty \frac{m+1}{(m + 1)!} \mbs{\phi}^{\wdg^m} \tau = \sum_{m=0}^\infty \frac{1}{m!} \mbs{\phi}^{\wdg^m} \tau = \mbf{C}(\mbs{\phi}) \tau.
\end{equation}
The first term of the $(1,3)$ block we have seen before and so what remains to show is the second term,
\begin{multline}
\BG_0(\mbs{\phi}^\wdg, \mbs{\nu}^\wdg, \mbs{\phi}^\wdg, \tau)  = \left(\Bg_0 \left(\mbs{\phi}^\Wdg, \mbs{\nu} \right) - \Bg_1 \left(\mbs{\phi}^\Wdg, \mbs{\nu} \right) \right)^\wdg \BG_0(\mbs{\phi}^\wdg) \tau = \left(\left(\BG_1 \left(\mbs{\phi}^\wdg\right) - \BG_2 \left(\mbs{\phi}^\wdg \right) \right) \mbs{\nu} \right)^\wdg \BG_0(\mbs{\phi}^\wdg) \tau \\ = \left((\mbf{J}(\mbs{\phi}) - \mbf{N}(\mbs{\phi}) )\mbs{\nu} \right)^\wdg \mbf{C}(\mbs{\phi}) \tau,
\end{multline}
which follows directly from Lemma~\ref{thm:bbinit2} in Appendix~\ref{sec:ellzero} and Remark~\ref{thm:remark1}.

Next, we assume the relationship holds for $\ell$ and show that it must therefore hold for $\ell + 1$.  Using Lemma~\ref{thm:bbrecursive}  we have
\begin{multline}
\BG_{\ell+1}(\mbs{\xi}^\Wdg) = \int_0^1 \alpha^\ell \, \BG_\ell(\alpha \mbs{\xi}^\Wdg) \, d\alpha \\ 
\hspace*{-0.83in}= \scalebox{0.8}{$\int_0^1 \alpha^\ell \bbm \BG_\ell(\alpha\mbs{\phi}^\wdg) & -\BG_\ell(\alpha\mbs{\phi}^\wdg, \alpha\tau) & \BG_\ell(\alpha\mbs{\phi}^\wdg, \alpha\mbs{\rho}^\wdg, \alpha\mbs{\phi}^\wdg) - \BG_\ell(\alpha\mbs{\phi}^\wdg, \alpha\mbs{\nu}^\wdg, \alpha\mbs{\phi}^\wdg, \alpha\tau) & \Bg_\ell(\alpha\mbs{\phi}^\wdg, \alpha\mbs{\nu}) \\ \mbf{0} & \BG_\ell(\alpha\mbs{\phi}^\wdg) & \BG_\ell(\alpha\mbs{\phi}^\wdg, \alpha\mbs{\nu}^\wdg, \alpha\mbs{\phi}^\wdg) & \mbf{0} \\ \mbf{0} & \mbf{0} &  \BG_\ell(\alpha\mbs{\phi}^\wdg) & \mbf{0} \\ \mbf{0}^T & \mbf{0}^T & \mbf{0}^T & \frac{1}{\ell!} \ebm $}  \\ 
\hspace*{0.1in} = \scalebox{0.7}{$\bbm \int_0^1 \alpha^\ell \BG_\ell(\alpha\mbs{\phi}^\wdg)\, d\alpha & -\int_0^1 \alpha^\ell \BG_\ell(\alpha\mbs{\phi}^\wdg, \alpha\tau)\, d\alpha & \int_0^1 \alpha^\ell \BG_\ell(\alpha\mbs{\phi}^\wdg, \alpha\mbs{\rho}^\wdg,\alpha\mbs{\phi}^\wdg)\, d\alpha - \int_0^1 \alpha^\ell \BG_\ell(\alpha\mbs{\phi}^\wdg, \alpha\mbs{\nu}^\wdg, \alpha\mbs{\phi}^\wdg, \alpha\tau)\, d\alpha & \int_0^1 \alpha^\ell \Bg_\ell(\alpha\mbs{\phi}^\wdg, \alpha\mbs{\nu})\, d\alpha \\ \mbf{0} & \int_0^1 \alpha^\ell \BG_\ell(\alpha\mbs{\phi}^\wdg)\, d\alpha & \int_0^1 \alpha^\ell \BG_\ell(\alpha\mbs{\phi}^\wdg, \alpha\mbs{\nu}^\wdg, \alpha\mbs{\phi}^\wdg)\, d\alpha & \mbf{0} \\ \mbf{0} & \mbf{0} &  \int_0^1 \alpha^\ell \BG_\ell(\alpha\mbs{\phi}^\wdg)\, d\alpha & \mbf{0} \\ \mbf{0}^T & \mbf{0}^T & \mbf{0}^T & \frac{\int_0^1 \alpha^\ell \,d\alpha}{\ell !} \ebm$} \\ 
=  \bbm \BG_{\ell + 1}(\mbs{\phi}^\wdg) & -\BG_{\ell + 1}(\mbs{\phi}^\wdg, \tau) & \BG_{\ell + 1}(\mbs{\phi}^\wdg, \mbs{\rho}^\wdg, \mbs{\phi}^\wdg) - \BG_{\ell + 1}(\mbs{\phi}^\wdg, \mbs{\nu}^\wdg, \mbs{\phi}^\wdg, \tau) & \Bg_{\ell + 1}(\mbs{\phi}^\wdg, \mbs{\nu}) \\ \mbf{0} & \BG_{\ell + 1}(\mbs{\phi}^\wdg) & \BG_{\ell + 1}(\mbs{\phi}^\wdg, \mbs{\nu}^\wdg, \mbs{\phi}^\wdg) & \mbf{0} \\ \mbf{0} & \mbf{0} &  \BG_{\ell + 1}(\mbs{\phi}^\wdg) & \mbf{0} \\ \mbf{0}^T & \mbf{0}^T & \mbf{0}^T & \frac{1}{(\ell+1)!} \ebm.
\end{multline}
Therefore the result holds for all $\ell$.
\end{proof}

Again, an immediate application of Theorem~\ref{thm:bbadgal3so3} is that we can compute the Jacobian of $\mbox{Ad}(SGal(3))$, which is given by
\begin{equation}
\Jbig(\mbs{\xi}) = \BG_1(\mbs{\xi}^\Wdg) = \bbm \BG_1(\mbs{\phi}^\wdg) & -\BG_1(\mbs{\phi}^\wdg, \tau) & \BG_1(\mbs{\phi}^\wdg, \mbs{\rho}^\wdg, \mbs{\phi}^\wdg) - \BG_1(\mbs{\phi}^\wdg, \mbs{\nu}^\wdg, \mbs{\phi}^\wdg, \tau) & \Bg_1(\mbs{\phi}^\wdg, \mbs{\nu}) \\ \mbf{0} & \BG_1(\mbs{\phi}^\wdg) & \BG_1(\mbs{\phi}^\wdg, \mbs{\nu}^\wdg, \mbs{\phi}^\wdg) & \mbf{0} \\ \mbf{0} & \mbf{0} &  \BG_1(\mbs{\phi}^\wdg) & \mbf{0} \\ \mbf{0}^T & \mbf{0}^T & \mbf{0}^T & 1 \ebm.
\end{equation}
We have seen how to compute $\BG_1(\mbs{\phi}^\wdg)$, $\BG_1(\mbs{\phi}^\wdg, \mbs{\rho}^\wdg, \mbs{\phi}^\wdg)$, and $\BG_1(\mbs{\phi}^\wdg, \mbs{\nu}^\wdg, \mbs{\phi}^\wdg)$ previously in the $SE(3)$ and $SE_2(3)$ cases, while the other blocks still require some work.  For the $(1,4)$ block we see that
\begin{multline}\label{eq:SO3_Nnu}
\Bg_1(\mbs{\phi}^\wdg, \mbs{\nu}) = \BG_2(\mbs{\phi}^\wdg) \mbs{\nu} = \mbf{N}(\mbs{\phi}) \mbs{\nu} = \int_0^1 \alpha \mbf{J}(\alpha \mbs{\phi}) \, d\alpha \, \mbs{\nu} = \int_0^1 \alpha \left( j_0(\alpha\phi)  \mbf{1} + \alpha j_1(\alpha\phi) \mbs{\phi}^\wdg + \alpha^2 j_2(\alpha\phi) \mbs{\phi}^{\wdg^2} \right) \, d\alpha \, \mbs{\nu} \\ 
= \left( \left( \int_0^1 \alpha \, d\alpha \right) \mbf{1} +  \left( \int_0^1 \alpha^2 j_1(\alpha\phi) \, d\alpha \right) \mbs{\phi}^\wdg + \left( \int_0^1 \alpha^3 j_2(\alpha\phi) \, d\alpha \right) \mbs{\phi}^{\wdg^2} \right) \mbs{\nu} \\
= \biggl( \underbrace{\frac{1}{2}}_{n_0(\phi)} \mbf{1} + \underbrace{\frac{\phi - \sin\phi}{\phi^3}}_{n_1(\phi)} \mbs{\phi}^\wdg + \underbrace{\frac{\phi^2 + 2 \cos\phi - 2}{2\phi^4}}_{n_2(\phi)}  \mbs{\phi}^{\wdg^2} \biggr) \mbs{\nu},
\end{multline}
where we make note of $n_0(\phi)$, $n_1(\phi)$, and $n_2(\phi)$ for use below.
For the $(1,2)$ block we have
\begin{multline}\label{eq:SGal3_12}
\BG_1(\mbs{\phi}^\wdg, \tau) = \int_0^1 \BG_0(\alpha\mbs{\phi}^\wdg, \alpha\tau) \, d\alpha = \int_0^1 \alpha \BG_0(\alpha\mbs{\phi}^\wdg) \, d\alpha \, \tau \\ = \int_0^1 \alpha \left( c_0(\alpha\phi)\mbf{1} + \alpha c_1(\alpha\phi) \mbs{\phi}^\wdg + \alpha^2 c_2(\alpha\phi) \mbs{\phi}^{\wdg^2} \right) \, d\alpha \, \tau \hspace*{1.6in} \\
= \left( \left( \int_0^1 \alpha \, d\alpha \right) \mbf{1} +  \left( \int_0^1 \alpha^2 c_1(\alpha\phi) \, d\alpha \right) \mbs{\phi}^\wdg + \left( \int_0^1 \alpha^3 c_2(\alpha\phi) \, d\alpha \right) \mbs{\phi}^{\wdg^2} \right) \tau \\
= \left( \frac{1}{2} \mbf{1} + \frac{\sin\phi - \phi \cos\phi}{\phi^3} \mbs{\phi}^\wdg + \frac{\phi^2 - 2\phi\sin\phi - 2\cos\phi + 2}{2\phi^4} \mbs{\phi}^{\wdg^2} \right) \tau.
\end{multline}
Finally, for the $(1,3)$ block we already know how to compute the first term so it remains to compute $\BG_1(\mbs{\phi}^\wdg, \mbs{\nu}^\wdg, \mbs{\phi}^\wdg, \tau)$.  We will once again employ our recursive relationship and write
\begin{equation}
\BG_1(\mbs{\phi}^\wdg, \mbs{\nu}^\wdg, \mbs{\phi}^\wdg, \tau) = \int_0^1 \BG_0(\alpha\mbs{\phi}^\wdg, \alpha\mbs{\nu}^\wdg, \alpha\mbs{\phi}^\wdg, \alpha\tau) \, d\alpha. 
\end{equation}
We will do well to calculate $\BG_0(\mbs{\phi}^\wdg, \mbs{\nu}^\wdg, \mbs{\phi}^\wdg, \tau)$ explicitly first:
\begin{multline}
\BG_0(\mbs{\phi}^\wdg, \mbs{\nu}^\wdg, \mbs{\phi}^\wdg, \tau) = \left((\mbf{J}(\mbs{\phi}) - \mbf{N}(\mbs{\phi}) )\mbs{\nu} \right)^\wdg \mbf{C}(\mbs{\phi}) \tau \\
 = \left( \left( \left( j_0(\phi) - n_0(\phi) \right) \mbf{1} + \left( j_1(\phi) - n_1(\phi) \right) \mbs{\phi}^\wdg + \left( j_2(\phi) - n_2(\phi) \right) \mbs{\phi}^{\wdg^2} \right) \mbs{\nu} \right)^\wdg \\ \times \;\; \left(  c_0(\phi) \mbf{1} + c_1(\phi) \mbs{\phi}^\wdg + c_2(\phi) \mbs{\phi}^{\wdg^2}  \right) \tau.
\end{multline}
Applying Lemma~\ref{thm:product} in Appendix~\ref{sec:identities} to expand and then simplifying some of the coefficients, we find
\begin{multline}
\BG_0(\mbs{\phi}^\wdg, \mbs{\nu}^\wdg, \mbs{\phi}^\wdg, \tau) = \Biggl( \frac{1}{2} \mbs{\nu}^\wdg  + \left(j_1(\phi) - n_1(\phi)\right) \mbs{\phi}^\wdg\mbs{\nu}^\wdg + \left( j_1(\phi) + n_1(\phi) - c_2(\phi) \right) \mbs{\nu}^\wdg \mbs{\phi}^\wdg + \left( j_2(\phi) - n_2(\phi) \right) \mbs{\phi}^{\wdg^2} \mbs{\nu}^\wdg \\ + \left( j_1(\phi) - j_2(\phi) - \frac{1}{2} c_2(\phi) \right) \mbs{\phi}^\wdg \mbs{\nu}^\wdg \mbs{\phi}^\wdg + n_2(\phi) \, \mbs{\nu}^\wdg \mbs{\phi}^{\wdg^2} \\ + \frac{1}{2}\left( \frac{1}{2} n_1(\phi) - n_2(\phi) \right) \left( \mbs{\phi}^{\wdg^2} \mbs{\nu}^\wdg \mbs{\phi}^{\wdg} + \mbs{\phi}^\wdg \mbs{\nu}^\wdg \mbs{\phi}^{\wdg^2} \right) \biggr) \tau.
\end{multline}
And so now finally 
\begin{multline}\label{eq:SGal3_13b}
\BG_1(\mbs{\phi}^\wdg, \mbs{\nu}^\wdg, \mbs{\phi}^\wdg, \tau) = \int_0^1 \BG_0(\alpha\mbs{\phi}^\wdg, \alpha\mbs{\nu}^\wdg, \alpha\mbs{\phi}^\wdg, \alpha\tau) \, d\alpha \\
= \Biggl( \frac{1}{2} \left( \int_0^1 \alpha^2 \, d\alpha \right) \mbs{\nu}^\wdg + \left( \int_0^1 \alpha^3 \left(j_1(\alpha\phi) - n_1(\alpha\phi)\right) d\alpha \right) \mbs{\phi}^\wdg\mbs{\nu}^\wdg \hspace*{1.25in} \\  
+ \left( \int_0^1 \alpha^3 \left( j_1(\alpha\phi) + n_1(\alpha\phi) - c_2(\alpha\phi) \right)  d\alpha \right) \mbs{\nu}^\wdg \mbs{\phi}^\wdg 
+ \left( \int_0^1 \alpha^4 \left( j_2(\alpha\phi) - n_2(\alpha\phi) \right) d\alpha \right) \mbs{\phi}^{\wdg^2} \mbs{\nu}^\wdg \\ + \left( \int_0^1 \alpha^4  \left( j_1(\alpha\phi) - j_2(\alpha\phi) - \frac{1}{2} c_2(\alpha\phi) \right)   d\alpha \right) \mbs{\phi}^\wdg \mbs{\nu}^\wdg \mbs{\phi}^\wdg 
+ \left( \int_0^1 \alpha^4  n_2(\alpha\phi) d\alpha \right) \mbs{\nu}^\wdg \mbs{\phi}^{\wdg^2} \\
+ \frac{1}{2} \left( \int_0^1 \alpha^5 \left( \frac{1}{2} n_1(\alpha\phi) - n_2(\alpha\phi) \right)  d\alpha \right) \left( \mbs{\phi}^{\wdg^2} \mbs{\nu}^\wdg \mbs{\phi}^{\wdg} + \mbs{\phi}^\wdg \mbs{\nu}^\wdg \mbs{\phi}^{\wdg^2} \right) \Biggr) \tau \\
= \Biggl( \frac{1}{6} \mbs{\nu}^\wdg + \frac{2-\phi\sin\phi - 2\cos\phi}{\phi^4}\mbs{\phi}^\wdg\mbs{\nu}^\wdg + \frac{\phi^2 + 2\cos\phi - 2}{2\phi^4}  \mbs{\nu}^\wdg \mbs{\phi}^\wdg + \frac{\phi^3 + 6\phi + 6\phi\cos\phi - 12 \sin\phi}{6\phi^5} \mbs{\phi}^{\wdg^2} \mbs{\nu}^\wdg \\
+ \frac{12\sin\phi - 12\phi\cos\phi - 3\phi^2\sin\phi - \phi^3}{6\phi^5} \mbs{\phi}^\wdg \mbs{\nu}^\wdg \mbs{\phi}^\wdg + \frac{\phi^3 - 6\phi + 6\sin\phi}{6\phi^5} \mbs{\nu}^\wdg \mbs{\phi}^{\wdg^2} \\ 
+ \frac{4+ \phi^2 + \phi^2\cos\phi  - 4\phi\sin\phi - 4\cos\phi}{4\phi^6} \left( \mbs{\phi}^{\wdg^2} \mbs{\nu}^\wdg \mbs{\phi}^{\wdg} + \mbs{\phi}^\wdg \mbs{\nu}^\wdg \mbs{\phi}^{\wdg^2} \right) \biggr) \tau,
\end{multline}
which exactly matches the results of \citet{kellyAllGalileanGroup2024a} and \citet{fornasierEquivariantSymmetriesAided2024} upon noting that $\mbs{\phi}^{\wdg^2} \mbs{\nu}^\wdg \mbs{\phi}^{\wdg} = \mbs{\phi}^\wdg \mbs{\nu}^\wdg \mbs{\phi}^{\wdg^2}$ and $\mbs{\phi}^{\wdg^2} \mbs{\nu}^\wdg \mbs{\phi}^{\wdg^2} = -\phi^2 \mbs{\phi}^\wdg \mbs{\nu}^\wdg \mbs{\phi}^{\wdg}$.  However, in our case we came to the result without expanding infinite series and recognizing patterns therein; rather, we computed seven integrals.


\subsection{Similarity Transformations, $Sim(3)$}

The group of similarity transformations, $Sim(3)$, extends $SE(3)$ in a different way by incorporating an unknown scale parameter, $\lambda$.  This group is used in computer vision problems.  We follow \citet{eadeLieGroupsComputer2014, eadeLieGroups2D2018} for the background.

Our pose matrix is still $4 \times 4$ but modified to include the scale parameter, $\lambda$:
\begin{equation}
\Tsmall(\mbs{\xi}) = \exp\left(\mbs{\xi}^\wdg \right) = \exp\left( \bbm \mbs{\rho} \\ \mbs{\phi} \\ \lambda \ebm^\wdg \right) = \exp\left( \bbm \mbs{\phi}^\wdg & \mbs{\rho} \\ \mbf{0}^T & -\lambda \ebm \right) = \bbm \mbf{C}(\mbs{\phi}) & \mbf{r}(\mbs{\phi},\mbs{\rho},\lambda) \\ \mbf{0}^T & s(\lambda)^{-1} \ebm,
\end{equation}
where $s(\lambda) = \exp(\lambda)$ and
\begin{equation}\label{eq:sim3r}
\mbf{r}(\mbs{\phi},\mbs{\rho},\lambda) = \sum_{m=0}^\infty \sum_{n=0}^\infty \frac{1}{(m+n+1)!} \mbs{\phi}^{\wdg^m} \, \mbs{\rho} \, (-\lambda)^n.
\end{equation}
We have implicitly defined a new version of $\wdg$ for $Sim(3)$.  The $7 \times 7$ adjoint for $Sim(3)$ is given by
\begin{multline}
\Tbig(\mbs{\xi}) = \exp\left(\mbs{\xi}^\Wdg \right) = \exp\left( \bbm \mbs{\rho} \\ \mbs{\phi} \\ \lambda \ebm^\Wdg \right) = \exp\left(\bbm \mbs{\phi}^\wdg + \lambda \mbf{1} & \mbs{\rho}^\wdg & -\mbs{\rho} \\ \mbf{0} & \mbs{\phi}^\wdg & \mbf{0} \\ \mbf{0}^T & \mbf{0}^T & 1 \ebm \right) \\
= \bbm \mbf{C}(\mbs{\phi})\, s(\lambda) & \mbf{r}(\mbs{\phi},\mbs{\rho},\lambda)^\wdg \mbf{C}(\mbs{\phi}) \, s(\lambda) & -\mbf{r}(\mbs{\phi},\mbs{\rho}, \lambda) \, s(\lambda) \\ \mbf{0} & \mbf{C}(\mbs{\phi}) & \mbf{0} \\ \mbf{0}^T & \mbf{0}^T & 1 \ebm
\end{multline}
where we have defined a new version of $\Wdg$ for $Sim(3)$.

In general, $Sim(3)$ appears to be less explored in the literature than some of the other groups that we have investigated.  Nevertheless, we can still offer our usual theorems.

\begin{theorem}\label{thm:bbsim3so3}
$Sim(3)$: Given the building-block functions in~\eqref{eq:bb}, we have that
\begin{equation}
\BG_\ell(\mbs{\xi}^\wdg) = \bbm \BG_\ell(\mbs{\phi}^\wdg) & \Bg_\ell(\mbs{\phi}^\wdg, \mbs{\rho}, -\lambda) \\ \mbf{0}^T & \bg_\ell(-\lambda) \ebm.
\end{equation}
In other words, for all $\ell$ the $Sim(3)$ building blocks can be constructed from the $SO(3)$ building blocks.
\end{theorem}
\begin{proof}
Following our proof-by-induction model, we first consider the $\ell = 0$ case.  We must show
\begin{equation}
\BG_0(\mbs{\xi}^\wdg) = \bbm \BG_0(\mbs{\phi}^\wdg) & \Bg_0(\mbs{\phi}^\wdg, \mbs{\rho}, -\lambda) \\ \mbf{0}^T & \bg_0(-\lambda) \ebm = \bbm \mbf{C}(\mbs{\phi}) & \mbf{r}(\mbs{\phi},\mbs{\rho},\lambda) \\ \mbf{0}^T & s(\lambda)^{-1} \ebm = \Tsmall(\mbs{\xi}).
\end{equation}
The top-left block is just the series definition of the rotation matrix that we have seen several times.  For the top-right block, our definition of $\mbf{r}(\mbs{\phi},\mbs{\rho},\lambda)$ in~\eqref{eq:sim3r} matches the definition of $\Bg_0(\mbs{\phi}^\wdg, \mbs{\rho}, -\lambda)$ in~\eqref{eq:bb}.  The bottom-right entry is also straightforward:  $\bg_0(-\lambda) = \sum_{m=0}^\infty \frac{1}{m!} (-\lambda)^m = \exp(-\lambda) = \exp(\lambda)^{-1} = s(\lambda)^{-1}$. 

We assume the relationship hold for $\ell$ and show it holds for $\ell + 1$.  Using Lemma~\ref{thm:bbrecursive}:
\begin{multline}
\BG_{\ell+1}(\mbs{\xi}^\wdg) = \int_0^1 \alpha^\ell \BG_\ell(\alpha\mbs{\xi}^\wdg) \, d\alpha = \int_0^1 \alpha^\ell \bbm \BG_\ell(\alpha\mbs{\phi}^\wdg) & \Bg_0(\alpha\mbs{\phi}^\wdg, \alpha\mbs{\rho}, -\alpha\lambda) \\ \mbf{0}^T & \bg_0(-\alpha\lambda) \ebm \, d\alpha \\ = \bbm \int_0^1 \alpha^\ell\BG_\ell(\alpha\mbs{\phi}^\wdg) \, d\alpha & \int_0^1 \alpha^\ell\Bg_0(\alpha\mbs{\phi}^\wdg, \alpha\mbs{\rho}, -\alpha\lambda) \, d\alpha \\ \mbf{0}^T & \int_0^1 \alpha^\ell\bg_0(-\alpha\lambda) \, d\alpha \ebm = \bbm \BG_{\ell+1}(\mbs{\phi}^\wdg) & \Bg_{\ell+1}(\mbs{\phi}^\wdg, \mbs{\rho}, -\lambda) \\ \mbf{0}^T & \bg_{\ell+1}(-\lambda) \ebm.
\end{multline}
\end{proof}
\smallskip

\begin{theorem}\label{thm:bbadsim3so3}
$\mbox{\normalfont Ad}(Sim(3))$: Given the building-block functions in~\eqref{eq:bb}, we have that
\begin{equation}
\BG_\ell(\mbs{\xi}^\Wdg) = \bbm \BG_\ell(\mbs{\phi}^\wdg + \lambda\mbf{1}) & \BG_\ell(\mbs{\phi}^\wdg + \lambda\mbf{1}, \mbs{\rho}^\wdg, \mbs{\phi}^\wdg) & - \Bg_\ell(\mbs{\phi}^\wdg + \lambda\mbf{1}, \mbs{\rho}) \\ \mbf{0} & \BG_\ell(\mbs{\phi}^\wdg) & \mbf{0} \\ \mbf{0}^T & \mbf{0}^T & \frac{1}{\ell !} \ebm.
\end{equation}
In other words, for all $\ell$ the $\mbox{\normalfont Ad}(Sim(3))$ building blocks can be constructed from the $SO(3)$ building blocks.
\end{theorem}
\begin{proof}
Following our proof-by-induction model, we first consider the $\ell = 0$ case.  We must show
\begin{multline}
\BG_0(\mbs{\xi}^\Wdg) = \bbm \BG_0(\mbs{\phi}^\wdg + \lambda\mbf{1}) & \BG_0(\mbs{\phi}^\wdg + \lambda\mbf{1}, \mbs{\rho}^\wdg, \mbs{\phi}^\wdg) & - \Bg_0(\mbs{\phi}^\wdg + \lambda\mbf{1}, \mbs{\rho}) \\ \mbf{0} & \BG_0(\mbs{\phi}^\wdg) & \mbf{0} \\ \mbf{0}^T & \mbf{0}^T & 1 \ebm \\ = \bbm \mbf{C}(\mbs{\phi})\, s(\lambda) & \mbf{r}(\mbs{\phi},\mbs{\rho},\lambda)^\wdg \mbf{C}(\mbs{\phi}) \, s(\lambda) & -\mbf{r}(\mbs{\phi},\mbs{\rho}, \lambda) \, s(\lambda) \\ \mbf{0} & \mbf{C}(\mbs{\phi}) & \mbf{0} \\ \mbf{0}^T & \mbf{0}^T & 1 \ebm = \Tbig(\mbs{\xi}).
\end{multline}
For the $(1,1)$ block we have
\begin{equation}
\BG_0(\mbs{\phi}^\wdg + \lambda\mbf{1}) = \sum_{m=0}^\infty \frac{1}{m!} (\mbs{\phi}^\wdg + \lambda\mbf{1})^m = \exp(\mbs{\phi}^\wdg + \lambda\mbf{1}) = \exp\left(\mbs{\phi}^\wdg \right) \exp(\lambda) = \mbf{C}(\mbs{\phi})\, s(\lambda).
\end{equation}
For the $(1,3)$ block we can use Remark~\ref{thm:remark1} and later Lemma~\ref{thm:intfact1} to write
\begin{multline}\label{eq:sim3r2}
\Bg_0(\mbs{\phi}^\wdg + \lambda\mbf{1}, \mbs{\rho}) = \BG_1(\mbs{\phi}^\wdg + \lambda\mbf{1}) \mbs{\rho} = \int_0^1 \BG_0(\alpha(\mbs{\phi}^\wdg + \lambda\mbf{1})) \, d\alpha \, \mbs{\rho} = \int_0^1 \exp(\alpha(\mbs{\phi}^\wdg + \lambda\mbf{1})) \, d\alpha \, \mbs{\rho} \\
= \int_0^1 \exp(\alpha\mbs{\phi}^\wdg )\exp(\alpha\lambda) \, d\alpha \, \mbs{\rho} = \int_0^1 \exp(\alpha\mbs{\phi}^\wdg )\exp((\alpha-1)\lambda) \, d\alpha \, \exp(\lambda) \mbs{\rho} \\
= \int_0^1 \left(\sum_{m=0}^\infty \frac{1}{m!} (\alpha\mbs{\phi})^{\wdg^m} \right)\left(\sum_{n=0}^\infty \frac{1}{n!} ((1-\alpha)(-\lambda))^n \right) \, d\alpha \, \mbs{\rho} \, s(\lambda) \\ = \sum_{m=0}^\infty \sum_{n=0}^\infty \underbrace{\frac{1}{m!\,n!}  \int_0^1 \alpha^m(1-\alpha)^n \, d\alpha}_{\mbox{\scriptsize use Lemma~\ref{thm:intfact1}}} \, \mbs{\phi}^{\wdg^m} (-\lambda)^n \, \mbs{\rho} \, s(\lambda) \\ 
= \underbrace{\sum_{m=0}^\infty \sum_{n=0}^\infty \frac{1}{(m+n+1)!} \mbs{\phi}^{\wdg^m} \, \mbs{\rho} \,  (-\lambda)^n}_{\mbf{r}(\mbs{\phi},\mbs{\rho},\lambda)}  \, s(\lambda) = \mbf{r}(\mbs{\phi},\mbs{\rho},\lambda) \, s(\lambda).
\end{multline}
For the $(1,2)$ block we can use Lemma~\ref{thm:bbinit1} in Appendix~\ref{sec:ellzero} then the result from the $(1,3)$ block to have
\begin{equation}
\BG_0(\mbs{\phi}^\wdg + \lambda\mbf{1}, \mbs{\rho}^\wdg, \mbs{\phi}^\wdg) = \underbrace{\Bg_0(\mbs{\phi}^\wdg + \lambda\mbf{1}, \mbs{\rho})}_{\mbf{r}(\mbs{\phi},\mbs{\rho}, \lambda) s(\lambda)} \!\mbox{}^\wdg \BG_0(\mbs{\phi}^\wdg) = \mbf{r}(\mbs{\phi},\mbs{\rho},\lambda)^\wdg \mbf{C}(\mbs{\phi}) \, s(\lambda).
\end{equation}
The $(2,2)$ block is simply the definition of the rotation matrix.
\end{proof}

Suppose now we would like a compact analytic expression for $\Tsmall(\mbs{\xi}) = \BG_0(\mbs{\xi}^\wdg)$.  We know how to calculate all the entries except for $\mbf{r}(\mbs{\phi},\mbs{\rho},\lambda)$.  We will need $\BG_0\left(\mbs{\phi}^\wdg + \lambda\mbf{1}\right)$, which we can see is 
\begin{equation}\label{eq:Cs}
\BG_0\left(\mbs{\phi}^\wdg + \lambda\mbf{1}\right) = \exp\left(\mbs{\phi}^\wdg\right) \exp(\lambda) = \exp(\lambda) \, \mbf{1} + \frac{\exp(\lambda)\sin\phi}{\phi} \mbs{\phi}^\wdg +\frac{ \exp(\lambda) (1 - \cos\phi)}{\phi^2} \mbs{\phi}^{\wdg^2}. 
\end{equation}
We can then use the following logic where the first step is a re-arrangement of~\eqref{eq:sim3r2}:
\begin{multline}\label{eq:Sim3_r}
\mbf{r}(\mbs{\phi},\mbs{\rho},\lambda) = \Bg_0\left(\mbs{\phi}^\wdg + \lambda\mbf{1}, \mbs{\rho}\right) s(\lambda)^{-1} = \BG_1\left(\mbs{\phi}^\wdg + \lambda\mbf{1}\right) \mbs{\rho} \, s(\lambda)^{-1} = \int_0^1 \BG_0\left(\alpha(\mbs{\phi}^\wdg + \lambda\mbf{1})\right) \, d\alpha \, \mbs{\rho} \, s(\lambda)^{-1} \\
= \left( \left(\int_0^1 \exp(\alpha\lambda) \, d\alpha\right) \mbf{1} + \frac{1}{\phi}\left(\int_0^1 \exp(\alpha\lambda)\sin(\alpha\phi) \, d\alpha\right) \mbs{\phi}^\wdg + \left(\int_0^1 \exp(\alpha\lambda)(1-\cos(\alpha\phi)) \, d\alpha\right)  \mbs{\phi}^{\wdg^2}\right) \mbs{\rho} \, s(\lambda)^{-1}\\
= \Biggl( \underbrace{\frac{1-s(\lambda)^{-1}}{\lambda}}_{m_0(\phi,\lambda)} \mbf{1} + \underbrace{\frac{\phi s(\lambda)^{-1} + \lambda\sin\phi - \phi\cos\phi}{\phi(\lambda^2 + \phi^2)}}_{m_1(\phi,\lambda)} \mbs{\phi}^\wdg + \underbrace{\left(\frac{\lambda -\phi\sin\phi-\lambda\cos\phi}{\phi^2(\lambda^2+\phi^2)} + \frac{1-s(\lambda)^{-1}}{\lambda(\lambda^2+\phi^2)} \right)}_{m_2(\phi,\lambda)}  \mbs{\phi}^{\wdg^2} \Biggr) \mbs{\rho} \\ = \mbf{M}(\mbs{\phi},\lambda) \mbs{\rho},
\end{multline}
where we have defined $\mbf{M}(\mbs{\phi},\lambda)$ for convenience.
This result matches exactly that of \citep{eadeLieGroupsComputer2014}, which was computed by a lengthy series formulation; our approach involved three simple integrals.

Moreover, with this result in hand we can compute the adjoint, $\Tbig(\mbs{\xi})$.  The $(1,1)$ block is simply
\begin{equation}
\mbf{C}(\mbs{\phi}) \, s(\lambda) = \BG_0\left(\mbs{\phi}^\wdg + \lambda\mbf{1}\right),
\end{equation}
whose formula appears in~\eqref{eq:Cs}.  For the $(1,3)$ block we see that
\begin{equation}
\mbf{r}(\mbs{\phi},\mbs{\rho},\lambda) \, s(\lambda) = s(\lambda) \mbf{M}(\mbs{\phi},\lambda) \mbs{\rho}= \left( s(\lambda)m_0(\phi,\lambda) \mbf{1} + s(\lambda) m_1(\phi,\lambda) \mbs{\phi} + s(\lambda) m_2(\phi,\lambda) \mbs{\phi}^{\wdg^2} \right) \mbs{\rho}.
\end{equation}
Finally, for the $(1,2)$ block we can form
\begin{multline}
\mbf{r}(\mbs{\phi},\mbs{\rho},\lambda)^\wdg \mbf{C}(\mbs{\phi}) \, s(\lambda) = \left( s(\lambda) \mbf{M}(\phi,\lambda) \mbs{\rho}\right)^\wdg \mbf{C}(\mbs{\phi}) \\
= \left( \left( s(\lambda)m_0(\phi,\lambda) \mbf{1} + s(\lambda) m_1(\phi,\lambda) \mbs{\phi} + s(\lambda) m_2(\phi,\lambda) \mbs{\phi}^{\wdg^2} \right) \mbs{\rho} \right)^\wdg \left( c_0(\phi) \mbf{1} + c_1(\phi) \mbs{\phi} + c_2(\phi) \mbs{\phi}^{\wdg^2} \right)
\end{multline}
and then use Lemma~\ref{thm:product} in Appendix~\ref{sec:identities} to expand and group the terms.  The $(2,2)$ block is just the usual rotation matrix.

We leave it as an exercise for the reader to compute the Jacobian of $Sim(3)$.  Naturally, we can use Theorem~\ref{thm:bbadsim3so3} to write
\begin{equation}
\Jbig(\mbs{\xi}) = \int_0^1 \Tbig(\alpha\mbs{\xi}) \, d\alpha,
\end{equation}
and then integrate one block at a time.  The $(1,2)$ block will be the most onerous.


\subsection{Rotations and Poses in the Plane, $SO(2)$ and $SE(2)$}

Planar rigid motion, $SE(2)$, is a subgroup of $SE(3)$, but there are some subtleties worth spelling out.

\subsubsection{Rotations, $SO(2)$}

We first briefly review the details of $SO(2)$, which represents just scalar rotation, as it is a bit of a unique case.  As mentioned in the introduction, the Lie algebra is made up of the set of all $2 \times 2$ skew-symmetric matrices of the form
\begin{equation}
\phi^\wdg = \phi \, \mbf{S}, \quad \mbf{S} = \bbm 0 & -1 \\ 1 & 0 \ebm,
\end{equation}
where $\mbf{S}$ is the canonical skew-symmetric matrix.  The minimal polynomial is then
\begin{equation}
\phi^{\wdg^2} + \phi^2 \mbf{1} = \mbf{0}.
\end{equation}
Rotation matrices are then constructed in the usual way:
\begin{equation}\label{eq:SO2_0}
\mbf{C}(\phi) = \sum_{m=0}^\infty \frac{1}{m!} \phi^{\wdg^m} = \underbrace{\left( 1 - \frac{1}{2!} \phi^2 + \frac{1}{4!} \phi^4 + \cdots \right)}_{\cos\phi} \mbf{1} + \underbrace{\left( \phi - \frac{1}{3!} \phi^3 + \frac{1}{5!} \phi^5 + \cdots \right)}_{\sin\phi} \mbf{S} = \bbm \cos\phi & -\sin\phi \\ \sin\phi & \cos\phi \ebm
\end{equation}
The adjoint, $\mathcal{C}(\phi)$, is degenerate since
\begin{equation}
\mbf{C}(\phi_1) \phi_2^\wdg \mbf{C}(\phi_1)^T  = \phi_2^\wdg = \left( \mathcal{C}(\phi_1) \phi_2 \right)^\wdg \quad \Rightarrow \quad \mathcal{C}(\phi_1) = 1.
\end{equation}
The Jacobian, $\mathcal{J}(\phi)$, is then also degenerate since
\begin{equation}
\mathcal{J}(\phi) = \int_0^1 \mathcal{C}(\alpha\phi) \, d\alpha = \int_0^1 d\alpha = 1.
\end{equation}
This is perhaps not surprising because we have that
\begin{equation}
\exp\left( \phi_1^\wdg + \phi_2^\wdg \right) = \exp\left( \phi_1^\wdg \right) \exp\left( \phi_2^\wdg \right),
\end{equation}
which follows from the Baker-Campbell-Hausdorff formula when the Lie bracket is zero:
\begin{equation}
\left[ \phi_1^\wdg, \phi_2^\wdg \right] = \phi_1\phi_2 ( \mbf{S} \mbf{S} - \mbf{S}\mbf{S}) = \mbf{0}.
\end{equation}
In other words, we can simply add the angles to compound rotations since their Lie algebra vectors are parallel.
The kinematics also simplify due to the identity Jacobian:
\begin{equation}
\dot{\mbf{C}}(\phi) = \omega^\wdg \mbf{C} \quad \Rightarrow \quad \dot{\phi} = \mathcal{J}(\phi)^{-1} \omega = \omega.
\end{equation}
When we extend to $SE(2)$, we will have to be careful to observe that for $SO(2)$ we have $\mathcal{J}(\phi) \neq \bg_1(\phi)$.

\subsubsection{Poses, $SE(2)$}

The group of planar poses, $SE(2)$, has a $3 \times 3$ form:
\begin{equation}\label{eq:SE2_0}
\Tsmall(\mbs{\xi}) = \exp\left( \mbs{\xi}^\wdg \right) = \exp\left( \bbm \mbs{\rho} \\ \phi \ebm^\wdg \right) = \exp\left( \bbm \phi \mbf{S} & \mbs{\rho} \\ \mbf{0}^T & 0 \ebm
\right) = \bbm \mbf{C}(\phi) & \BG_1(\phi^\wdg) \mbs{\rho} \\ \mbf{0}^T & 1 \ebm = \bbm \cos\phi & -\sin\phi & x \\ \sin\phi & \cos\phi & y \\ 0 & 0 & 1 \ebm,
\end{equation}
where $\mbs{\rho}$ is $2 \times 1$.  We note there is a subtle but important difference in the top-right block when compared to $SE(3)$: in $SE(2)$ we must write $\Bg_1(\phi^\wdg)$, which is not equivalent to the Jacobian of $SO(2)$, $\mathcal{J}(\phi)$. The group of adjoint poses, $\mbox{Ad}(SE(2))$, also has a $3 \times 3$ form:
\begin{equation}\label{eq:AdSE2_0}
\Tbig(\mbs{\xi}) = \exp\left( \mbs{\xi}^\Wdg \right) = \exp\left( \bbm \mbs{\rho} \\ \phi \ebm^\Wdg \right) = \exp\left( \bbm \phi \mbf{S} & -\mbf{S}\mbs{\rho} \\ \mbf{0}^T & 0 \ebm
\right) = \bbm \mbf{C}(\phi) & -\mbf{S}\, \BG_1(\phi^\wdg) \mbs{\rho} \\ \mbf{0}^T & 1 \ebm.
\end{equation}
We can check the adjoint relationship as follows:
\begin{multline}
\Tsmall(\mbs{\xi}_1) \, \mbs{\xi}_2^\wdg \, \Tsmall(\mbs{\xi})^{-1} = \bbm \mbf{C}(\phi_1) & \BG_1(\phi_1^\wdg) \mbs{\rho}_1 \\ \mbf{0}^T & 1 \ebm \bbm \phi_2 \mbf{S} & \mbs{\rho}_2 \\ \mbf{0}^T & 0  \ebm \bbm \mbf{C}(\phi_1) & -\mbf{C}(\phi_1)\BG_1(\phi_1^\wdg) \mbs{\rho}_1 \\ \mbf{0}^T & 1 \ebm 
\\
= \bbm \phi_2 \mbf{S} & \mbf{C}(\phi_1) \mbs{\rho}_2 - \phi_2 \mbf{S}\, \BG_1(\phi_1^\wdg) \mbs{\rho}_1 \\ \mbf{0}^T & 0 \ebm = \bbm \mbf{C}(\phi_1) \mbs{\rho}_2 - \phi_2 \mbf{S}\, \BG_1(\phi_1^\wdg) \mbs{\rho}_1 \\ \phi_2 \ebm^\wdg  \\
 = \left(  \bbm \mbf{C}(\phi_1) & -\mbf{S}\, \BG_1(\phi_1^\wdg) \mbs{\rho}_1 \\ \mbf{0}^T & 1 \ebm \bbm \mbs{\rho}_2 \\ \phi_2 \ebm \right)^\wdg = \left( \Tbig(\mbs{\xi}_1) \, \mbs{\xi}_2 \right)^\wdg,
\end{multline}
where we employ the useful identity, $\mbf{S} \,\BG_\ell(\phi^\wdg) = \BG_\ell(\phi^\wdg) \, \mbf{S}$.  

From here we can add two more substructure results to our list:
\begin{corollary}\label{thm:bbse2so2}
$SE(2)$:  Similarly to Theorem~\ref{thm:bbse3so3}, given the building-block functions in~\eqref{eq:bb} and, we have that
\begin{equation}
\BG_\ell(\mbs{\xi}^\wdg) = \bbm \BG_\ell(\phi^\wdg) & \Bg_\ell(\phi^\wdg, \mbs{\rho}) \\ \mbf{0}^T & \frac{1}{\ell !} \ebm.
\end{equation}
In other words, for all $\ell$ the $SE(2)$ building blocks can be constructed from the $SO(2)$ building blocks.
\end{corollary}

\begin{corollary}\label{thm:bbadse2so2}
$\mbox{\normalfont Ad}(SE(2))$:  Similarly to Theorem~\ref{thm:bbse3so3}, given the building-block functions in~\eqref{eq:bb}, we have that
\begin{equation}
\BG_\ell(\mbs{\xi}^\Wdg) = \bbm \BG_\ell(\phi^\wdg) & - \mbf{S} \Bg_\ell(\phi^\wdg, \mbs{\rho}) \\ \mbf{0}^T & \frac{1}{\ell !} \ebm.
\end{equation}
In other words, for all $\ell$ the $\mbox{Ad}(SE(2))$ building blocks can be constructed from the $SO(2)$ building blocks.
\end{corollary}

We omit the proofs of these corollaries as they are similar to previous ones.

We can turn to our favourite application as a demonstration, computing the Jacobian of $\mbox{Ad}(SE(2))$.  From Corollary~\ref{thm:bbadse2so2} and Lemma~\ref{thm:bbrecursive} we have
\begin{equation}
\Jbig(\mbs{\xi}) = \BG_1(\mbs{\xi}^\Wdg) = \bbm \BG_1(\phi^\wdg) & - \mbf{S} \Bg_1(\phi^\wdg, \mbs{\rho}) \\ \mbf{0}^T & 1 \ebm = \bbm \BG_1(\phi^\wdg) & - \mbf{S} \, \BG_2(\phi^\wdg) \mbs{\rho} \\ \mbf{0}^T & 1 \ebm.
\end{equation}
We compute $\BG_1(\phi^\wdg)$ and $\BG_2(\phi^\wdg)$ recursively:
\begin{subequations}
\begin{gather}
\BG_1(\phi^\wdg) = \int_0^1 \BG_0(\alpha\phi^\wdg) \, d\alpha = \left( \int_0^1 \cos(\alpha\phi) \, d\alpha \right) \mbf{1} + \left( \int_0^1 \sin(\alpha\phi) \, d\alpha \right) \mbf{S} = \frac{\sin\phi}{\phi}\mbf{1} + \frac{1-\cos\phi}{\phi}\mbf{S}, \label{eq:SO2_1} \\
\BG_2(\phi^\wdg) = \int_0^1 \alpha\BG_1(\alpha\phi^\wdg) \, d\alpha = \left( \int_0^1 \alpha \frac{\sin(\alpha\phi)}{\alpha\phi} \, d\alpha \right) \mbf{1} + \left( \alpha\int_0^1 \frac{1-\cos(\alpha\phi)}{\alpha\phi} \, d\alpha \right) \mbf{S} = \frac{1-\cos\phi}{\phi^2}\mbf{1} + \frac{\phi-\sin\phi}{\phi^2}\mbf{S}. \label{eq:SO2_2}
\end{gather}
\end{subequations}

\section{Derivatives}\label{sec:derivatives}

While we have made much of our {\em integral} forms there are also some related results involving derivatives and our building blocks worth mentioning.  Firstly, looking at Lemma~\ref{thm:bbrecursive}, it should not be surprising that we can go in the other direction using a derivative.  For example, we have
\begin{equation}
\bg_\ell(x) = \left. \frac{d}{d\alpha} \left( \alpha^{\ell+1} \bg_{\ell +1}(\alpha x) \right) \right|_{\alpha=1}.
\end{equation}
Similar results hold for the other building blocks.

In general, kinematics over a Lie group will have the form
\begin{equation}
\dot{\BG}_0(\mbf{x}^\wdg) = \mbf{v}^\wdg \BG_0(\mbf{x}^\wdg),
\end{equation}
where $\mbf{v}$ is the generalized velocity.  Isolating for $\mbf{v}^\wdg$ we see that
\begin{multline}
\mbf{v}^\wdg = \dot{\BG}_0(\mbf{x}^\wdg) \BG_0(\mbf{x}^\wdg)^{-1} = \int_0^1  \BG_0(\alpha\mbf{x}^\wdg) \, \dot{\mbf{x}}^\wdg \, \BG_0((1-\alpha)\mbf{x}^\wdg) \, d\alpha \, \BG_0(\mbf{x}^\wdg)^{-1} = \int_0^1  \BG_0(\alpha\mbf{x}^\wdg) \, \dot{\mbf{x}}^\wdg \, \BG_0(-\alpha\mbf{x}^\wdg) \, d\alpha \\
\int_0^1 \left(  \BG_0(\alpha\mbf{x}^\Wdg) \, \dot{\mbf{x}} \right)^\wdg \, d\alpha = \left( \int_0^1 \BG_0(\alpha\mbf{x}^\Wdg) \, d\alpha \, \dot{\mbf{x}} \right)^\wdg = \left( \BG_1(\mbf{x}^\Wdg)\, \dot{\mbf{x}} \right)^\wdg,
\end{multline}
which implies that $\mbf{v} = \BG_1(\mbf{x}^\Wdg)\, \dot{\mbf{x}} = \Bg_0(\mbf{x}^\Wdg, \dot{\mbf{x}})$, in terms of our building blocks (e.g., see \citet[p.21]{parkOptimalKinematicDesign1991}).  From here, we can claim a more general result using our building blocks.

\begin{theorem}\label{thm:kin}
Given the building blocks in~\eqref{eq:bb}, we can compute the temporal derivative of $\BG_\ell(\mbf{x}^\wdg)$ as
\begin{equation}
\dot{\BG}_\ell(\mbf{x}^\wdg) = \BG_\ell(\mbf{x}^\wdg,\dot{\mbf{x}}^\wdg, \mbf{x}^\wdg).
\end{equation} 
\end{theorem}
\begin{proof}
We use a proof by induction again.  For the $\ell=0$ case we have
\begin{equation}
\dot{\BG}_0(\mbf{x}^\wdg) = \mbf{v}^\wdg \BG_0(\mbf{x}^\wdg) =  \Bg_0(\mbf{x}^\Wdg, \dot{\mbf{x}})^\wdg  \BG_0(\mbf{x}^\wdg) = \BG_0(\mbf{x}^\wdg,\dot{\mbf{x}}^\wdg, \mbf{x}^\wdg),
\end{equation}
where we employed Lemma~\ref{thm:bbinit1} from Appendix~\ref{sec:ellzero} in the last step.  Now we assume the relation holds for $\ell$ and show this implies it holds for $\ell +1$:
\begin{multline}
\dot{\BG}_{\ell+1}(\mbf{x}^\wdg) = \frac{d}{dt} \BG_{\ell+1}(\mbf{x}^\wdg) = \frac{d}{dt} \int_0^1 \alpha^\ell \BG_\ell(\alpha\mbf{x}^\wdg) \, d\alpha = \int_0^1 \alpha^\ell \frac{d}{dt} \BG_\ell(\alpha\mbf{x}^\wdg) \, d\alpha \\ = \int_0^1 \alpha^\ell \BG_\ell(\alpha\mbf{x}^\wdg,\alpha\dot{\mbf{x}}^\wdg, \alpha\mbf{x}^\wdg) \, d\alpha = \BG_{\ell+1}(\mbf{x}^\wdg,\dot{\mbf{x}}^\wdg,\mbf{x}^\wdg),
\end{multline}
concluding the proof.
\end{proof}
\begin{corollary}\label{thm:adkin}
Following on to Theorem~\ref{thm:kin}, there is an adjoint version:
\begin{equation}
\dot{\BG}_\ell(\mbf{x}^\Wdg) = \BG_\ell(\mbf{x}^\Wdg,\dot{\mbf{x}}^\Wdg,\mbf{x}^\Wdg).
\end{equation}
\end{corollary}
\begin{proof}
Similar to proof of Theorem~\ref{thm:kin} but requiring that we use the adjoint kinematics:
\begin{equation}
\dot{\BG}_0(\mbf{x}^\Wdg) = \mbf{v}^\Wdg \BG_0(\mbf{x}^\Wdg).
\end{equation}
with $\mbf{v} = \BG_1(\mbf{x}^\Wdg)\, \dot{\mbf{x}} = \Bg_0(\mbf{x}^\Wdg, \dot{\mbf{x}})$.   Also requires a generalization of  Lemma~\ref{thm:bbinit1} in Appendix~\ref{sec:ellzero} to the adjoint case:  $\BG_0(\mbf{x}^\Wdg,\dot{\mbf{x}}^\Wdg,\mbf{x}^\Wdg) = \Bg_0(\mbf{x}^\Wdg, \dot{\mbf{x}})^\Wdg  \BG_0(\mbf{x}^\Wdg)$.
\end{proof}

This can be useful when we are interested in higher-order derivatives in our kinematics.  For example, we have
\begin{equation}
\dot{\mbf{v}} = \frac{d}{dt} \left( \BG_1(\mbf{x}^\wdg)\, \dot{\mbf{x}} \right) =  \dot{\BG}_1(\mbf{x}^\wdg)\, \dot{\mbf{x}} + \BG_1(\mbf{x}^\wdg)\, \ddot{\mbf{x}} = \BG_1(\mbf{x}^\wdg, \dot{\mbf{x}}^\wdg,\mbf{x}^\wdg)\, \dot{\mbf{x}} + \BG_1(\mbf{x}^\wdg)\, \ddot{\mbf{x}},
\end{equation}
as a simple demonstration.

Finally we state a result involving partial derivatives with respect to a Lie algebra vector, generalizing results of \citet[App.B.]{barfoot_ser24}.

\begin{theorem}\label{thm:partialderiv}
Given the building blocks in~\eqref{eq:bb}, we can compute the partial derivative of $\Bg_0(\mbf{x}^\Wdg,\mbf{y})$ with respect to $\mbf{x}$ as
\begin{equation}
\frac{\partial \Bg_0(\mbf{x}^\Wdg,\mbf{y})}{\partial \mbf{x}} = \BG_1(\mbf{x}^\Wdg,\mbf{y}^\Wdg,\mbf{x}^\Wdg) - \Bg_0(\mbf{x}^\Wdg,\mbf{y})^\Wdg \BG_1(\mbf{x}^\Wdg).
\end{equation}
\end{theorem}
\begin{proof}
We manipulate the building blocks as follows:
\begin{multline}
\frac{\partial \Bg_0(\mbf{x}^\Wdg,\mbf{y})}{\partial \mbf{x}} = \frac{\partial}{\partial \mbf{x}} \BG_1(\mbf{x}^\Wdg) \mbf{y} = \frac{\partial}{\partial \mbf{x}} \int_0^1 \BG_0(\alpha\mbf{x}^\Wdg) \, d\alpha \, \mbf{y} \\ =  \int_0^1 \frac{\partial}{\partial \mbf{x}} \left( \BG_0(\alpha\mbf{x}^\Wdg) \mbf{y}\right) \, d\alpha  = - \int_0^1\left( \BG_0(\alpha\mbf{x}^\Wdg) \alpha  \mbf{y} \right)^\Wdg \BG_1(\alpha\mbf{x}^\Wdg) \, d\alpha \hspace*{1.9in} \\ = \left. - \left( \BG_1(\alpha\mbf{x}^\Wdg) \alpha\mbf{y}\right)^\Wdg \alpha \BG_1(\alpha\mbf{x}^\Wdg)  \right|_{\alpha=0}^1 + \int_0^1  \left( \BG_1(\alpha\mbf{x}^\Wdg) \alpha\mbf{y}\right)^\Wdg \BG_0(\alpha\mbf{x}^\Wdg) \, d\alpha \hspace*{1.45in} \\ = \int_0^1  \BG_0\left(\alpha\mbf{x}^\Wdg, \alpha\mbf{y}^\Wdg, \alpha\mbf{x}^\Wdg\right) \, d\alpha - \left( \BG_1(\mbf{x}^\Wdg) \mbf{y}\right)^\Wdg  \BG_1(\mbf{x}^\Wdg) = \BG_1(\mbf{x}^\Wdg,\mbf{y}^\Wdg, \mbf{x}^\Wdg) - \Bg_0(\mbf{x}^\Wdg,\mbf{y})^\Wdg \BG_1(\mbf{x}^\Wdg),
\end{multline}
as required.
\end{proof}

As an example of Theorem~\ref{thm:partialderiv}, consider the case of $SO(3)$ (recall $\wdg = \Wdg$) when computing the partial of angular velocity with respect to the Lie algebra vector for rotation:
\begin{multline}
\frac{\partial \mbs{\omega}(\mbs{\phi},\dot{\mbs{\phi}})}{\partial \mbs{\phi}} = \frac{\partial}{\partial \mbs{\phi}} \left( \mbf{J}(\mbs{\phi}) \dot{\mbs{\phi}} \right) = \frac{\partial}{\partial \mbs{\phi}} \Bg_0(\mbs{\phi}^\Wdg,\dot{\mbs{\phi}}) = \BG_1(\mbs{\phi}^\Wdg,\dot{\mbs{\phi}}^\Wdg,\mbs{\phi}^\Wdg) - \Bg_0(\mbs{\phi}^\Wdg,\dot{\mbs{\phi}})^\Wdg \BG_1(\mbs{\phi}^\Wdg) \\ = \dot{\BG}_1(\mbs{\phi}^\Wdg) - \Bg_0(\mbs{\phi}^\Wdg,\dot{\mbs{\phi}})^\wdg \BG_1(\mbs{\phi}^\Wdg)  = \dot{\mbf{J}}(\mbs{\phi}) - \mbs{\omega}(\mbs{\phi},\dot{\mbs{\phi}})^\wdg \mbf{J}(\mbs{\phi}),
\end{multline}
which is a well-known result in attitude dynamics \citep{hughes86}.  Or, alternatively, consider taking the derivative of the translational component of $SE(3)$ with respect to the angular component of the Lie algebra vector:
\begin{equation}
\frac{\partial \mbf{r}(\mbs{\phi},\mbs{\rho})}{\partial \mbs{\phi}} = \frac{\partial}{\partial \mbs{\phi}} \left( \mbf{J}(\mbs{\phi}) \mbs{\rho} \right) = \frac{\partial}{\partial \mbs{\phi}} \Bg_0(\mbs{\phi}^\Wdg,\mbs{\rho}) = \BG_1(\mbs{\phi}^\Wdg,\mbs{\rho}^\Wdg,\mbs{\phi}^\Wdg) - \Bg_0(\mbs{\phi}^\Wdg,\mbs{\rho})^\Wdg \BG_1(\mbs{\phi}^\Wdg) = \mbf{Q}(\mbs{\phi},\mbs{\rho}) - \mbf{r}(\mbs{\phi},\mbs{\rho})^\wdg \mbf{J}(\mbs{\phi}),
\end{equation}
which is again recognizable in terms of our building blocks.\footnote{This identity was conjectured to the author by Frank Dellaert in a personal communication.}

\section{Conclusion}

\begin{table}[p]
\caption{Summary of matrix Lie group substructures.  Where equations are too long for the table, references to equation numbers are provided.}
\label{tbl:substructure}
\centering
\scalebox{0.90}{
\begin{tabular}{|c|c|c|}
\hline Lie Group & Key Quantities  & Additional Notes \\ \hline \hline
$SO(2)$ & $\mbf{C}(\phi) = \cos\phi \, \mbf{1} + \sin\phi \, \mbf{S}$ & $\mbf{S} = \bbm 0 & -1 \\ 1 & 0 \ebm$ \\
$\mbox{Ad}(SO(2))$ & $\mathcal{C}(\phi) = 1$ &  \\
Jacobian & $\mathcal{J}(\phi) = 1$ &  \\ \hline
$SE(2)$ & $\Tsmall(\mbs{\xi}) = \bbm \mbf{C}(\phi) & \BG_1(\phi^\wdg) \mbs{\rho} \\ \mbf{0}^T & 1 \ebm$ & $\mbs{\xi} = \bbm \mbs{\rho} \\ \phi \ebm$ \\
$\mbox{Ad}(SE(2))$ & $\Tbig(\mbs{\xi}) = \bbm \mbf{C}(\phi) & -\mbf{S}\, \BG_1(\phi^\wdg) \mbs{\rho} \\ \mbf{0}^T & 1 \ebm$ & $\BG_1(\phi^\wdg) = \mbox{\eqref{eq:SO2_1}}$  \\ 
Jacobian & $\Jbig(\mbs{\xi}) = \bbm \BG_1(\phi^\wdg) & - \mbf{S} \, \BG_2(\phi^\wdg) \mbs{\rho} \\ \mbf{0}^T & 1 \ebm$ & $\BG_2(\phi^\wdg) = \mbox{\eqref{eq:SO2_2}}$ \\ \hline \hline
$SO(3)$ & $\mbf{C}(\mbs{\phi}) =  \mbf{1} +  \frac{\sin\phi}{\phi}  \mbs{\phi}^\wdg + \frac{1-\cos\phi}{\phi^2} \mbs{\phi}^{\wdg^2}$ & \\ 
$\mbox{Ad}(SO(3))$ & $\mbf{C}(\mbs{\phi})$ & self-adjoint \\ 
Jacobian & $\mbf{J}(\mbs{\phi}) = \mbf{1} + \frac{1-\cos\phi}{\phi^2} \mbs{\phi}^\wdg + \frac{\phi - \sin\phi}{\phi^3} \mbs{\phi}^{\wdg^2}$ & \\ \hline
$SE(3)$ & $\Tsmall(\mbs{\xi}) = \bbm \mbf{C}(\mbs{\phi}) & \mbf{J}(\mbs{\phi}) \mbs{\rho} \\ \mbf{0}^T & 1 \ebm$ & $\mbs{\xi} = \bbm \mbs{\rho} \\ \mbs{\phi} \ebm$ \\
$\mbox{Ad}(SE(3))$ & $\Tbig(\mbs{\xi}) = \bbm \mbf{C}(\mbs{\phi}) & \left(\mbf{J}(\mbs{\phi}) \mbs{\rho}\right)^\wdg \mbf{C}(\mbs{\phi}) \\ \mbf{0} & \mbf{C}(\mbs{\phi}) \ebm = \mbox{\eqref{eq:SE3_monoT}}$ & \\
Jacobian & $\Jbig(\mbs{\xi}) = \bbm \mbf{J}(\mbs{\phi})  & \mbf{Q}(\mbs{\phi}, \mbs{\rho}) \\ \mbf{0}^T & \mbf{J}(\mbs{\phi}) \ebm = \mbox{\eqref{eq:SE3_monoJ}}$ & $\mbf{Q}(\mbs{\phi}, \mbs{\rho}) = \mbox{\eqref{eq:SE3_Q}}$ \\ \hline
$SE_2(3)$ & $\Tsmall(\mbs{\xi}) = \bbm \mbf{C}(\mbs{\phi}) & \mbf{J}(\mbs{\phi}) \mbs{\nu} & \mbf{J}(\mbs{\phi}) \mbs{\rho} \\ \mbf{0}^T & 1 & 0 \\ \mbf{0}^T & 0 & 1 \ebm$ & $\mbs{\xi} = \bbm \mbs{\rho} \\ \mbs{\nu} \\ \mbs{\phi} \ebm$ \\
$\mbox{Ad}(SE_2(3))$ & $\Tbig(\mbs{\xi}) = \bbm \mbf{C}(\mbs{\phi})  & \mbf{0} & (\mbf{J}(\mbs{\phi})\mbs{\rho})^\wdg \mbf{C}(\mbs{\phi}) \\ \mbf{0} & \mbf{C}(\mbs{\phi}) & (\mbf{J}(\mbs{\phi})\mbs{\nu})^\wdg  \mbf{C}(\mbs{\phi}) \\ \mbf{0} & \mbf{0} & \mbf{C}(\mbs{\phi}) \ebm$ &  \\
Jacobian & $\Jbig(\mbs{\xi}) = \bbm \mbf{J}(\mbs{\phi})  & \mbf{0} & \mbf{Q}(\mbs{\phi},\mbs{\rho}) \\ \mbf{0} & \mbf{J}(\mbs{\phi}) & \mbf{Q}(\mbs{\phi},\mbs{\nu})  \\ \mbf{0} & \mbf{0} & \mbf{J}(\mbs{\phi}) \ebm$ & $\mbf{Q}(\mbs{\phi}, \cdot) = \mbox{\eqref{eq:SE3_Q}}$ \\ \hline
$SGal(3)$ & $\Tsmall(\mbs{\xi}) = \bbm \mbf{C}(\mbs{\phi}) & \mbf{J}(\mbs{\phi}) \mbs{\nu} & \mbf{J}(\mbs{\phi}) \mbs{\rho} +   \mbf{N}(\mbs{\phi}) \mbs{\nu} \, \tau \\ \mbf{0}^T & 1 & \tau \\ \mbf{0}^T & 0 & 1 \ebm$ & $\mbs{\xi} = \bbm \mbs{\rho} \\ \mbs{\nu} \\ \mbs{\phi} \\ \tau \ebm$ \\
$\mbox{Ad}(SGal(3))$ & $\Tbig(\mbs{\xi}) = \bbm \mbf{C}(\mbs{\phi})  & -\mbf{C}(\mbs{\phi})\tau & \left(\mbf{J}(\mbs{\phi}) \mbs{\rho} -  (\mbf{J}(\mbs{\phi}) - \mbf{N}(\mbs{\phi}) )\mbs{\nu} \tau \right)^\wdg \mbf{C}(\mbs{\phi}) & \mbf{J}(\mbs{\phi}) \mbs{\nu} \\ \mbf{0} & \mbf{C}(\mbs{\phi}) & (\mbf{J}(\mbs{\phi})\mbs{\nu})^\wdg  \mbf{C}(\mbs{\phi}) & \mbf{0}  \\ \mbf{0} & \mbf{0} & \mbf{C}(\mbs{\phi}) & \mbf{0} \\ \mbf{0}^T & \mbf{0}^T & \mbf{0}^T & 1 \ebm$ & $\mbf{N}(\mbs{\phi}) = \mbox{\eqref{eq:SO3_Nnu}}$ \\
Jacobian & $\Jbig(\mbs{\xi}) =  \bbm \mbf{J}(\mbs{\phi}) & -\BG_1(\mbs{\phi}^\wdg, \tau) & \mbf{Q}(\mbs{\phi},\mbs{\rho}) - \BG_1(\mbs{\phi}^\wdg, \mbs{\nu}^\wdg, \mbs{\phi}^\wdg, \tau) & \mbf{N}(\mbs{\phi}) \mbs{\nu} \\ \mbf{0} & \mbf{J}(\mbs{\phi}) & \mbf{Q}(\mbs{\phi},\mbs{\nu}) & \mbf{0} \\ \mbf{0} & \mbf{0} &  \mbf{J}(\mbs{\phi}) & \mbf{0} \\ \mbf{0}^T & \mbf{0}^T & \mbf{0}^T & 1 \ebm$ & $\begin{matrix} \BG_1(\mbs{\phi}^\wdg, \tau) = \mbox{\eqref{eq:SGal3_12}} \\ \BG_1(\mbs{\phi}^\wdg, \mbs{\nu}^\wdg, \mbs{\phi}^\wdg, \tau) = \mbox{\eqref{eq:SGal3_13b}} \end{matrix}$ \\ \hline
$Sim(3)$ & $\Tsmall(\mbs{\xi}) = \bbm \mbf{C}(\mbs{\phi}) & \mbf{M}(\mbs{\phi},\lambda) \mbs{\rho} \\ \mbf{0}^T & s(\lambda)^{-1} \ebm$ & $\mbs{\xi} = \bbm \mbs{\rho} \\ \mbs{\phi} \\ \lambda \ebm$ \\
$\mbox{Ad}(Sim(3))$ & $\Tbig(\mbs{\xi}) = \bbm \mbf{C}(\mbs{\phi})\, s(\lambda) &  (\mbf{M}(\mbs{\phi},\lambda) \mbs{\rho} )^\wdg \mbf{C}(\mbs{\phi}) \, s(\lambda) & - \mbf{M}(\mbs{\phi},\lambda) \mbs{\rho}  \, s(\lambda) \\ \mbf{0} & \mbf{C}(\mbs{\phi}) & \mbf{0} \\ \mbf{0}^T & \mbf{0}^T & 1 \ebm$ & $\begin{matrix}  \mbf{M}(\mbs{\phi},\lambda) = \mbox{\eqref{eq:Sim3_r}} \\ s(\lambda) = \exp(\lambda) \end{matrix}$ \\ 
Jacobian & $\Jbig(\mbs{\xi}) = \int_0^1 \Tbig(\alpha\mbs{\xi}) \, d\alpha$ & left as an exercise \\ \hline
\end{tabular}
}
\end{table}

We have shown some new structure within matrix Lie groups based on the recursive integral relationship of Lemma~\ref{thm:bbrecursive}.  The ability to seamlessly transfer back and forth between series form and integral form means we can pick whichever tool is more amenable to the task at hand.  We studied several common groups and their substructures, recovering some hard-won results from the literature using our simpler integral approach; Table~\ref{tbl:substructure} summarizes these results.  In particular, we advocate for computing Jacobians block by block as this not only maintains the block-sparsity of the resulting matrix, but also modularizes the effort and keeps the calculations as simple as possible.  We believe that understanding the structure of matrix Lie groups in terms of our building blocks will generalize to other groups not studied here, particularly those built upon the rotation groups $SO(2)$ and $SO(3)$.  This may require further building blocks to be defined, but we hope the essence of our integral approach will still apply.

\section{Acknowledgements}

Several of the integrals were solved analytically using OpenAI's ChatGPT and Wolfram Alpha.  The author would like to thank Frank Dellaert, Alessandro Fornasier, and William Talbot for providing useful feedback on the paper.

\newpage
\appendix

\section{Lemmata}

This appendix provides some results used in the main body of the paper.  We locate them here to keep the main paper flow as clean as possible.

\subsection{Key Identities}\label{sec:identities}

There are a few key identities used elsewhere in the paper that we prove here.

\begin{lemma}\label{thm:intfact1}
For $m$ and $n$ whole numbers we have
\begin{equation}
\int_0^1 \alpha^m (1-\alpha)^n \, d\alpha = \frac{m!\,n!}{(m+n+1)!}.
\end{equation}
\end{lemma}
\begin{proof}
We can integrate by parts once as follows:
\begin{multline}
\int_0^1 \underbrace{\alpha^m}_{u} \underbrace{(1-\alpha)^n \, d\alpha}_{dv} = \biggl. \underbrace{\alpha^m}_{u} \underbrace{\left( -\frac{1}{n+1} (1-\alpha)^{n+1} \right)}_{v} \biggr|_{\alpha=0}^1 - \int_0^1  \underbrace{\left( -\frac{1}{n+1} (1-\alpha)^{n+1} \right)}_{v} \underbrace{m \alpha^{m-1} \, d\alpha}_{du} \\ = \frac{m}{n+1} \int_0^1 \alpha^{m-1} (1-\alpha)^{n+1} \, d\alpha.
\end{multline}
After $m$ such integrations by parts we have
\begin{equation}
\int_0^1 \alpha^m (1-\alpha)^n \, d\alpha = \frac{m \cdot (m-1) \cdots 2 \cdot 1}{(n+1) \cdot (n+2) \cdots (n+m-1) \cdot (n+m)} \underbrace{\int_0^1 \alpha^0 (1-\alpha)^{n+m} \, d\alpha}_{\frac{1}{n+m+1}} = \frac{m!\,n!}{(m+n+1)!},
\end{equation}
as required.
\end{proof}

There is another important identity of which we make use captured in the next lemma.
\begin{lemma}\label{thm:intfact2}
For $m$ and $n$ whole numbers we have
\begin{equation}
\int_0^1 \alpha \int_0^1 (\alpha\beta)^m (1-\alpha\beta)^n \, d\beta \, d\alpha = \frac{m!\,n!}{(m+n+1)!} - \frac{(m+1)!\,n!}{(m+n+2)!}.
\end{equation}
\end{lemma}
\begin{proof}
We begin by making a substitution of variables, replacing $\beta$ with $\gamma = \alpha \beta$ so that our integral becomes
\begin{equation}
\int_0^1 \int_0^\alpha \gamma^m (1-\gamma)^n \, d\gamma \, d\alpha.
\end{equation}
We then integrate by parts for the variable $\alpha$:
\begin{multline}
\int_0^1 \underbrace{\int_0^\alpha \gamma^m (1-\gamma)^n \, d\gamma}_{u} \, \underbrace{d\alpha}_{dv} = \biggl. \underbrace{\int_0^\alpha \gamma^m (1-\gamma)^n \, d\gamma}_{u} \underbrace{\alpha}_{v} \biggr|_{\alpha=0}^1 - \int_0^1 \underbrace{\alpha}_{v} \underbrace{\alpha^m (1-\alpha)^n \, d\alpha}_{du} \\ = \int_0^1 \gamma^m (1-\gamma)^n \, d\gamma - \int_0^1 \alpha^{m+1} (1-\alpha)^n \, d\alpha = \frac{m!\,n!}{(m+n+1)!} - \frac{(m+1)!\,n!}{(m+n+2)!},
\end{multline}
where the last step involves two applications of Lemma~\ref{thm:intfact1}.  The calculation of $du$ follows from the Fundamental Theorem of Calculus.
\end{proof}

The following multiplication of terms is also useful at times.
\begin{lemma}\label{thm:product}
The following $SO(3)$ expression holds:
\begin{multline}
\left( \left(a_0 \mbf{1} + a_1 \mbs{\phi}^\wdg + a_2 \mbs{\phi}^{\wdg^2} \right) \mbs{\rho} \right)^\wdg \left(b_0 \mbf{1} + b_1 \mbs{\phi}^\wdg + b_2 \mbs{\phi}^{\wdg^2} \right) = a_0b_0 \mbs{\rho}^\wdg + a_1b_0 \mbs{\phi}^\wdg \mbs{\rho}^\wdg \\ + \left(a_0b_1 - a_1b_0 + \phi^2 (a_1b_2 - a_2b_1)\right)) \mbs{\rho}^\wdg \mbs{\phi}^\wdg + a_2b_0 \mbs{\phi}^{\wdg^2}\mbs{\rho} + \left(a_0 b_2 - a_1b_1+a_2b_0 - \phi^2 a_2b_2 \right) \mbs{\rho}^\wdg \mbs{\phi}^{\wdg^2} \\
+\left(a_1b_1 - 2 a_2b_0 + \phi^2 a_2b_2 \right) \mbs{\phi}^\wdg \mbs{\rho}^\wdg \mbs{\phi}^\wdg + \frac{1}{2}\left( a_1b_2 - a_2b_1 \right) \left(\mbs{\phi}^{\wdg^2}\mbs{\rho}^\wdg\mbs{\phi}^\wdg + \mbs{\phi}^\wdg\mbs{\rho}^\wdg \mbs{\phi}^{\wdg^2}\right), 
\end{multline}
where $\mbs{\rho}$ is any compatible vector. 
\end{lemma}
\begin{proof}
We use that $\wdg$ is linear to multiply out then use the binomial-like expression,
\begin{equation}\label{eq:binomial}
\left( \mbs{\phi}^{\wdg^m} \mbs{\rho} \right)^\wdg = \sum_{k=0}^m (-1)^k {m \choose k} \mbs{\phi}^{\wdg^{m-k}} \mbs{\rho}^\wdg \mbs{\phi}^{\wdg^k},  
\end{equation}
to expand some of the terms.  Finally we use the minimal polynomial for $SO(3)$  and the identities $\mbs{\phi}^{\wdg^2} \mbs{\nu}^\wdg \mbs{\phi}^{\wdg} = \mbs{\phi}^\wdg \mbs{\nu}^\wdg \mbs{\phi}^{\wdg^2}$ and $\mbs{\phi}^{\wdg^2} \mbs{\nu}^\wdg \mbs{\phi}^{\wdg^2} = -\phi^2 \mbs{\phi}^\wdg \mbs{\nu}^\wdg \mbs{\phi}^{\wdg}$ to simplify some terms and group the expressions.
\end{proof}

\subsection{Some $\ell = 0$ Relationships}\label{sec:ellzero}

It will be important to establish a few key connections between our building blocks that are exploited in Section~\ref{sec:bb}.  One important (known) relationship is the usual adjoint map, which in terms of our building blocks is
\begin{equation}\label{eq:bbadjointmap}
\left( \BG_0(\mbf{x}^\Wdg) \mbf{y} \right)^\wdg = \BG_0(\mbf{x}^\wdg) \, \mbf{y}^\wdg \, \BG_0(\mbf{x}^\wdg)^{-1}, \quad \left( \BG_0(\mbf{x}^\Wdg) \mbf{y} \right)^\Wdg = \BG_0(\mbf{x}^\Wdg) \, \mbf{y}^\Wdg \, \BG_0(\mbf{x}^\Wdg)^{-1}.
\end{equation}
The following lemmata capture additional relationships.

\begin{lemma}\label{thm:bbinit1}
Given the building-block functions in~\eqref{eq:bb}, we have the following relationship
\begin{equation}
\BG_0(\mbf{x}^\wdg,\mbf{y}^\wdg,\mbf{x}^\wdg) = \Bg_0 \left(\mbf{x}^\Wdg, \mbf{y} \right)^\wdg \BG_0(\mbf{x}^\wdg).
\end{equation}
\end{lemma}
\begin{proof}
To see this, we use the following sequence of manipulations \citep[(8.65)]{barfoot_ser24}:
\begin{multline}
 \Bg_0 \left(\mbf{x}^\Wdg, \mbf{y} \right)^\wdg \BG_0(\mbf{x}^\wdg) = \left( \BG_1(\mbf{x}^\Wdg) \mbf{y}\right)^\wdg \BG_0(\mbf{x}^\wdg)  = \left( \int_0^1 \BG_0(\alpha\mbf{x}^\Wdg) \, d\alpha \, \mbf{y}\right)^\wdg \BG_0(\mbf{x}^\wdg)  = \int_0^1 \left( \BG_0(\alpha\mbf{x}^\Wdg) \mbf{y} \right)^\wdg \BG_0(\mbf{x}^\wdg)  \, d\alpha \\ 
= \int_0^1 \BG_0(\alpha\mbf{x}^\wdg)  \mbf{y}^\wdg \BG_0((1-\alpha)\mbf{x}^\wdg)  \, d\alpha = \!\int_0^1 \!\! \left( \sum_{m=0}^\infty \!\frac{1}{m!}\!\left(\alpha\mbf{x}^\wdg\right)^m \!\right) \!\mbf{y}^\wdg \!\! \left( \sum_{n=0}^\infty \!\frac{1}{n!}\!\left((1\!-\!\alpha)\mbf{x}^\wdg\right)^n \!\right) d\alpha\\
= \sum_{m=0}^\infty\sum_{n=0}^\infty \frac{1}{m!\,n!} \left( \int_0^1 \alpha^m(1\!-\!\alpha)^n  \,d\alpha \!\right) \mbf{x}^{\wdg^m} \mbf{y}^\wdg\mbf{x}^{\wdg^n},
\end{multline}
where we have used that $\wdg$ is linear and~\eqref{eq:bbadjointmap}.  Lemma~\ref{thm:intfact1} in Appendix~\ref{sec:identities} proves that 
\begin{equation}
\int_0^1 \alpha^m(1-\alpha)^n  \, d\alpha = \frac{m! \, n!}{(m+n+1)!},
\end{equation}
and therefore, $\BG_0(\mbf{x}^\wdg, \mbf{y}^\wdg, \mbf{x}^\wdg) = \Bg_0 \left(\mbf{x}^\Wdg, \mbf{y} \right)^\wdg \BG_0(\mbf{x}^\wdg)$, which is the desired result.
\end{proof}

Along similar lines we have the following.

\begin{lemma}\label{thm:bbinit2}
Given the building-block functions in~\eqref{eq:bb}, we have the following relationship
\begin{equation}
\BG_0(\mbf{x}^\wdg,\mbf{y}^\wdg,\mbf{x}^\wdg, \tau) = \left(\Bg_0 \left(\mbf{x}^\Wdg, \mbf{y} \right) - \Bg_1 \left(\mbf{x}^\Wdg, \mbf{y} \right) \right)^\wdg \BG_0(\mbf{x}^\wdg) \, \tau.
\end{equation}
\end{lemma}
\begin{proof}
We can break this into two terms where the first can be simplified by Lemma~\ref{thm:bbinit1} from Appendix~\ref{sec:ellzero}:
\begin{equation}
\left(\Bg_0 \left(\mbf{x}^\Wdg, \mbf{y} \right) - \Bg_1 \left(\mbf{x}^\Wdg, \mbf{y} \right) \right)^\wdg \BG_0(\mbf{x}^\wdg) \, \tau  =\underbrace{\Bg_0 \left(\mbf{x}^\Wdg, \mbf{y} \right)^\wdg \BG_0(\mbf{x}^\wdg)}_{\BG_0(\mbf{x}^\wdg,\mbf{y}^\wdg,\mbf{x}^\wdg)} \, \tau - \Bg_1 \left(\mbf{x}^\Wdg, \mbf{y} \right)^\wdg \BG_0(\mbf{x}^\wdg) \, \tau.
\end{equation}
For the second term, we can manipulate as follows:
\begin{multline}
\Bg_1 \left(\mbf{x}^\Wdg, \mbf{y} \right)^\wdg \BG_0(\mbf{x}^\wdg) = \left( \BG_2\left(\mbf{x}^\Wdg\right) \mbf{y} \right)^\wdg \BG_0(\mbf{x}^\wdg) = \int_0^1 \alpha \left( \BG_1(\alpha\mbf{x}^\Wdg) \mbf{y} \right)^\wdg \BG_0(\mbf{x}^\wdg) \, d\alpha \\ = \int_0^1 \alpha \int_0^1 \left( \BG_0(\alpha\beta\mbf{x}^\Wdg) \mbf{y} \right)^\wdg \BG_0(\mbf{x}^\wdg) \, d\beta\,d\alpha  = \int_0^1 \alpha \int_0^1 \BG_0(\alpha\beta\mbf{x}^\wdg) \, \mbf{y}^\wdg \, \BG_0((1-\alpha\beta)\mbf{x}^\wdg) \, d\beta\,d\alpha \\ = \sum_{m=0}^\infty \sum_{n=0}^\infty \frac{1}{m! \, n!} \left( \int_0^1 \alpha \int_0^1 (\alpha\beta)^m (1- \alpha\beta)^n \, d\beta \, d\alpha \right) \mbf{x}^{\wdg^m} \mbf{y}^\wdg \, \mbf{x}^{\wdg^n}.
\end{multline}
We can then use Lemma~\ref{thm:intfact2} in Appendix~\ref{sec:identities} to claim
\begin{equation}
\int_0^1 \alpha \int_0^1 (\alpha\beta)^m (1-\alpha\beta)^n \, d\beta \, d\alpha = \frac{m!\,n!}{(m+n+1)!} - \frac{(m+1)!\,n!}{(m+n+2)!}.
\end{equation}
Putting the pieces together we have that
\begin{multline}
\underbrace{\Bg_0 \left(\mbf{x}^\Wdg, \mbf{y} \right)^\wdg \BG_0(\mbf{x}^\wdg)}_{\BG_0(\mbf{x}^\wdg,\mbf{y}^\wdg,\mbf{x}^\wdg)} \, \tau  - \Bg_1 \left(\mbf{x}^\Wdg, \mbf{y} \right)^\wdg \BG_0(\mbf{x}^\wdg) \, \tau \\
= \sum_{m=0}^\infty \sum_{n=0}^\infty \left( \left( \frac{1}{(m+n+1)!} \right) - \left( \frac{1}{(m+n+1)!} - \frac{m+1}{(m+n+2)!}\right) \right) \mbf{x}^{\wdg^m} \mbf{y}^\wdg \mbf{x}^{\wdg^n} \, \tau  \\ = \sum_{m=0}^\infty \sum_{n=0}^\infty \frac{m+1}{(m+n+2)!} \mbf{x}^{\wdg^m} \mbf{y}^\wdg \mbf{x}^{\wdg^n} \, \tau = \BG_0(\mbf{x}^\wdg, \mbf{y}^\wdg, \mbf{x}^\wdg, \, \tau),
\end{multline}
the desired result.
\end{proof}


\end{document}